\documentclass[sigconf,nonacm]{acmart}
\AtBeginDocument{%
  }

\copyrightyear{2026}
\acmYear{2026}
\setcopyright{cc}
\setcctype{by}
\acmConference[KDD '26]{Proceedings of the 32nd ACM SIGKDD Conference on Knowledge Discovery and Data Mining V.2}{August 09--13, 2026}{Jeju Island, Republic of Korea}
\acmBooktitle{Proceedings of the 32nd ACM SIGKDD Conference on Knowledge Discovery and Data Mining V.2 (KDD '26), August 09--13, 2026, Jeju Island, Republic of Korea}
\acmDOI{10.1145/3770855.3818091}
\acmISBN{979-8-4007-2259-2/2026/08}

\usepackage{subcaption}
\usepackage{multirow}

\usepackage{xspace}

\usepackage{algorithm}
\usepackage{algorithmic}

\usepackage{mathtools}
\usepackage{amsthm}

\usepackage[capitalize,noabbrev]{cleveref}

\newcommand{\quotes}[1]{``#1''}

\newcommand{\ttdash}{-\allowbreak}

\theoremstyle{plain}
\newtheorem{theorem}{Theorem}[section]

\theoremstyle{definition}
\newtheorem{definition}[theorem]{Definition}
\newtheorem{assumption}[theorem]{Assumption}
\theoremstyle{remark}

\definecolor{AssocBlue}{HTML}{1F77B4}
\definecolor{UnrelOrange}{HTML}{D55E00}
\definecolor{FlipRed}{HTML}{C7362F}
\definecolor{OkGreen}{HTML}{228833}
\definecolor{Ink}{HTML}{2D2D2D}
\definecolor{Soft}{HTML}{F6F7FB}
\definecolor{RulePurple}{HTML}{9B51E0}

\definecolor{StageBlue}{HTML}{56CCF2}
\definecolor{StagePurple}{HTML}{9B51E0}
\definecolor{StageYellow}{HTML}{F2C94C}
\definecolor{StageOrange}{HTML}{F2994A}
\definecolor{StageMint}{HTML}{6FCF97}

\newcommand{\assocText}[1]{\textcolor{AssocBlue!85!black}{#1}}
\newcommand{\unrelText}[1]{\textcolor{UnrelOrange!55!black}{#1}}
\newcommand{\flipText}[1]{\textcolor{FlipRed!85!black}{#1}}
\newcommand{\preserveText}[1]{\textcolor{OkGreen!60!black}{#1}}
\newcommand{\ruleText}[1]{\textcolor{RulePurple!95!black}{#1}}
\newcommand{\Dpos}{\ensuremath{\textcolor{AssocBlue!85!black}{\mathcal{D}^{+}}}}
\newcommand{\Dneg}{\ensuremath{\textcolor{UnrelOrange!55!black}{\mathcal{D}^{-}}}}
\newcommand{\Dpm}{\ensuremath{\Dpos/\Dneg}}
\newcommand{\tDpos}{\ensuremath{\textcolor{AssocBlue!85!black}{\tilde{\mathcal{D}}^{+}}}}
\newcommand{\tDneg}{\ensuremath{\textcolor{UnrelOrange!55!black}{\tilde{\mathcal{D}}^{-}}}}

\newcommand{\rzero}{\ensuremath{\textcolor{RulePurple!95!black}{r_0}}}
\newcommand{\ranch}{\ensuremath{\textcolor{RulePurple!95!black}{r_j}}}

\newcommand{\AlgoName}{\textsc{MechaRule}\xspace}

\newcommand{\RepHook}{\texttt{ln\_final.\allowbreak hook\_normalized}}

\begin{document}

\title{Neuron-Anchored Rule Extraction for Large Language Models via Contrastive Hierarchical Ablation}

\author{Francesco Sovrano}
\orcid{0000-0002-6285-1041}
\affiliation{%
  \institution{Università della Svizzera italiana}
  \city{Lugano}
  \country{Switzerland}
}
\email{francesco.sovrano@usi.ch}

\author{Gabriele Dominici}
\orcid{0009-0009-1955-0778}
\affiliation{%
  \institution{Università della Svizzera italiana}
  \city{Lugano}
  \country{Switzerland}
}
\email{dominici.gab@gmail.com}

\author{Marc Langheinrich}
\orcid{0000-0002-8834-7388}
\affiliation{%
  \institution{Università della Svizzera italiana}
  \city{Lugano}
  \country{Switzerland}
}
\email{marc.langheinrich@usi.ch}

\renewcommand{\shortauthors}{F. Sovrano et al.}

\begin{abstract}
A central goal of explainable AI is to express large language model (LLM) decision logic symbolically and ground it in internal mechanisms. Existing rule-extraction methods usually learn ungrounded symbolic surrogates, while mechanistic interpretability links behavior to neurons but often requires hand-crafted hypotheses and costly interventions.
We introduce \AlgoName, a pipeline that grounds rule extraction in LLM circuits by localizing sparse \emph{agonist} activations whose ablation disrupts rule-related behavior. \AlgoName rests on two findings. First, in a fixed baseline/flip regime, sparse agonist effects can exhibit \emph{overtopping}: a few high-effect activations remain detectable within larger groups, dominate weaker ones, and flip many of the same examples. In such regimes, adaptive group testing with confidence-guided conservative pruning requires $O\!\left(k\log\frac{N}{k}+k\right)$ interventions over $N$ candidates when $k\ll N$ are agonists.
Second, agonists are localized more reliably on data splits aligned with close-to-faithful rule behavior; spectral splits provide a rule-free fallback, whereas unfaithful splits degrade localization.
Empirically, on arithmetic and jailbreaking, \AlgoName recalls 97.0\% of highest-effect agonists in matched brute-force validations at only 2.14\% of exhaustive-ablation cost on average. Ablating the localized agonists eliminates 97.6--100.0\% of eligible correct arithmetic answers and jailbreaks, and can correct arithmetic errors or induce jailbreaks by up to 72.8\% and 32.5\%.
\end{abstract}



\maketitle

\newcommand\kddavailabilitydoi{https://doi.org/10.5281/zenodo.20573178} 
\newcommand\kddrepositoryurl{https://github.com/Francesco-Sovrano/MechaRule}
\ifdefempty{\kddavailabilitydoi}{}{%
\begingroup
\small\noindent\raggedright\textbf{Resource Availability:} Source code and artifacts: \url{\kddavailabilitydoi}; repository: \url{\kddrepositoryurl}.
\endgroup
}

\section{Introduction} \label{sec:introduction}

Rule mining and interpretable prediction have been core topics in data science for decades \cite{DBLP:conf/sigmod/AgrawalIS93}, yet modern decision systems increasingly rely on large language models (LLMs) \cite{ren2024advancements,alenezi2025ai} whose behaviors are high-dimensional, context-dependent, and difficult to audit \cite{zhao2024explainability}. 

When LLMs succeed or fail on a class of inputs, stakeholders often want explanations that are both symbolic, hence inspectable and aligned with domain knowledge, and causal, hence testable by interventions for diagnosis and repair. Existing approaches usually provide only one of these properties: rule-extraction methods can describe input--output behavior compactly, but are typically ungrounded in model circuitry \cite{zilke2016deepred,sovrano2026ruleshap}, whereas mechanistic interpretability can localize causal components, but neuron-by-neuron probing is expensive and often depends on hand-crafted hypotheses and manual semantic labeling \cite{nikankin2024arithmetic,pan2024multimodalneurons}.

Recent work suggests that even on structured tasks such as arithmetic, LLMs may rely not on a single exact algorithm, but on many neuron-level heuristics \cite{lin2025implicit,yu2025genderbias,nikankin2024arithmetic}. For instance, \citet{nikankin2024arithmetic} report that a small subset of multilayer-perceptron (MLP) neurons explains much of arithmetic performance: these neurons fire for specific input patterns, such as operands ending in 0 or results in certain ranges, and nudge answer-token logits so accuracy emerges from their combined effects rather than from a general-purpose procedure.
This suggests that some task behaviors may be mediated by sparse, feature-selective internal activations, motivating our central question: can we \emph{efficiently} identify \emph{which} neurons are causally involved in a task-defined behavior, and \emph{what interpretable rule} predicts when each such neuron matters?

Formally, let $f_\theta$ be an LLM and let a task be specified by a dataset $\mathcal{D}$ of input--output pairs and a metric $m$ defining success. Let $B_{f_\theta}(x)\in\{0,1\}$ denote the induced \emph{task behavior}, i.e., whether $f_\theta$ succeeds under $m$ on example $x\in\mathcal{D}$. As shown in Figure~\ref{fig:problem_overview}, we seek an algorithm that:
\textit{(i)} localizes neurons whose targeted intervention flips $B_{f_\theta}$ on $\mathcal{D}$, and
\textit{(ii)} attaches to them a human-readable rule explaining \emph{when} the neuron controls $B_{f_\theta}$,
without exhaustive ablations or fully manual feature engineering.
We use \emph{neuron} operationally to mean an internal scalar activation in a chosen intervention basis.

\begin{figure}[t]
\centering
\includegraphics[width=\linewidth]{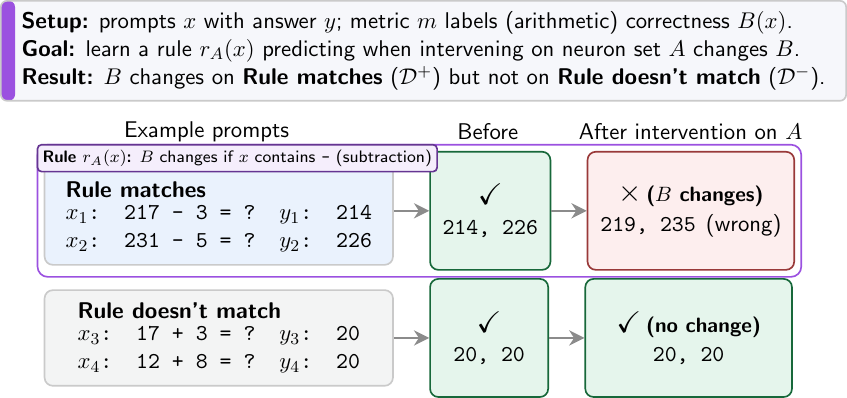}
\Description{Schematic illustrating the problem setup. Two groups of arithmetic prompts are shown: associated prompts containing subtraction (e.g., ``217 - 3 = ?'') and unrelated prompts containing addition (e.g., ``17 + 3 = ?''). Before intervention, answers are correct for both groups. After ablating an internal activation or small activation set, answers for the subtraction group change (becoming incorrect), while answers for the addition group remain unchanged. The figure defines a rule that predicts when behavior flips and labels the two slices as D+ (associated) and D- (unrelated).}
\caption{\textbf{Problem overview.} We seek an internal activation set whose ablation selectively \flipText{flips behavior} on $\Dpos$ while \preserveText{preserving behavior} on $\Dneg$; the anchored \ruleText{rule} $\ranch(x)$ predicts when that ablation flips an input.}
\label{fig:problem_overview}
\end{figure}

Our approach rests on two hypotheses, supported empirically and formalized under explicit assumptions.
\textbf{(H1) Sparse gating.} If LLM behaviors are implemented as a bag of heuristics over abstract features, then a faithful rule $r$ should correspond to a sparse set of internal neurons whose intervention has strong rule-conditional effects.
\textbf{(H2) Regime-conditional dominance and overlap under interventions.} Within a fixed baseline/flip regime, intervention effects can be monotone and saturating: influential neurons may flip overlapping example subsets, and one or a few dominant neurons may mask weaker ones at larger scales. We call this dominance pattern \emph{overtopping}.
%
Guided by H1--H2, we introduce \AlgoName, a \emph{task-local} method for neuron-anchored explanations that links symbolic rules to causal circuitry. Given a task $(\mathcal{D}, m)$ and a chosen intervention basis, \AlgoName\ returns \textit{(i)} a small set of causally influential neurons and \textit{(ii)} anchored rules predicting when those activations control the behavior. In our experiments, the basis consists of MLP-output write coordinates.

At a high level, \AlgoName\ proceeds in two phases, instantiated below as four stages.
First, it learns symbolic rules that separate success from failure on $\mathcal{D}$. It constructs a task-local, human-interpretable predicate vocabulary through an \emph{LLM-assisted predicate search}: an LLM inspects contrastive input--output pairs and proposes candidate predicates, such as \quotes{first operand divisible by 3} or \quotes{result in [175,200]}. The most discriminative predicates are retained and combined with a small set of manually defined seed predicates; RuleSHAP, a SHAP-based global rule-extraction method \cite{sovrano2026ruleshap}, then extracts sparse if--then rules.

Second, \AlgoName\ uses these rules to guide mechanistic localization. Rather than probing the full model, it restricts exact ablations to a behavior-relevant candidate set and identifies influential neurons with adaptive \emph{group ablations}. Neurons are organized into nested groups and ablated jointly. 
Because group effects can saturate or be dominated within a fixed regime, pruning does not rely on additivity; instead, it uses a high-confidence upper bound (UCB) to rule out groups whose estimated effect is below threshold, and descends whenever a group could still contain an agonist.
For each influential neuron, \AlgoName\ then learns a symbolic rule predicting when its intervention flips the model's decision, yielding a neuron-anchored description of an observable trigger for that behavior-changing intervention.

We evaluate \AlgoName\ on open-weight LLMs up to 7B parameters, i.e., \texttt{Qwen2\ttdash{}7B\ttdash{}Instruct}, \texttt{Qwen2\ttdash{}1.5B\ttdash{}Instruct}~\cite{yang2024qwen2}, and \texttt{GPT\ttdash{}J\ttdash{}6B}~\cite{wang2021gptj}. Our main tasks are arithmetic (following prior mechanistic setups for arithmetic behavior~\cite{nikankin2024arithmetic}) and Best-of-$N$ (BoN) jailbreaking (using an established jailbreak benchmark~\cite{hughes2024best}); Appendix~\ref{app:additional_results} reports an extra natural language inference (NLI) evaluation on HANS \cite{mccoy2019right}.

Across six model--task settings, \AlgoName\ localizes 1,749 \emph{agonists}, including 399 whose extracted anchored rules achieve Matthews correlation coefficient (MCC) $\ge0.70$ on held-out test cases. The union of localized agonist ablations flips 97.6\%--100.0\% of eligible target behavior in the original model: correct-to-incorrect in arithmetic and successful-to-unsuccessful jailbreaks in jailbreaking. In the reverse direction, the same union flips up to 72.8\% of eligible arithmetic errors into correct answers and up to 32.5\% of eligible baseline-safe jailbreak prompts into jailbreaks.
Overall, the results support a concrete conclusion: some task-local behaviors expose sparse, feature-selective causal handles, and successful anchors can summarize their triggers with interpretable rules.
%
Our main contributions are:
\begin{itemize}
    \item We formulate \emph{task-local neuron-anchored rule extraction}: learning symbolic rules grounded in causal interventions on internal components.
    \item We provide evidence for sparse gating (H1) and characterize regime-conditional dominance/overlap effects in group interventions (H2).
    \item We develop an overtopping-aware binary hierarchical group search with UCB-based conservative pruning and state its pruning-safety conditions per evaluation regime.
    \item We demonstrate anchored rule extraction and selective behavior control across models and tasks.
\end{itemize}

\section{Related Work} \label{sec:related_work}


A long line of work learns \emph{symbolic} rule sets or rule trees directly from tabular features, including associative classification \cite{liu1998integrating}, greedy rule learners such as RIPPER \cite{cohen1995fast}, Bayesian rule lists \cite{letham2015interpretable}, and CORELS \cite{angelino2017learning}. These methods yield compact, human-readable decision logic, but they assume an explicit feature representation and do not connect rules to the internal mechanisms of a large neural model.
%
Rule-extraction methods for neural networks either decompose small networks into axis-aligned rules (DeepRED \cite{zilke2016deepred}, ECLAIRE \cite{zarlenga2021efficient}) or fit a tree/ensemble surrogate to input--output behavior (RuleSHAP \cite{sovrano2026ruleshap}). 
Such approximations often scale poorly to highly complex models: as boundary complexity grows, faithfulness degrades unless the surrogate becomes complex or the analysis is local \cite{poyiadzi2021understanding,herbinger2023leveraging}.
In contrast, \AlgoName\ uses RuleSHAP only to propose symbolic structure, then anchors rules to neuron activations and validates them through targeted interventions, turning rule extraction from a descriptive surrogate into a mechanistically grounded account of internal components.

Several approaches use language models to generate natural-language descriptions of internal neurons or concepts, primarily as an observational interpretability layer. For example, \citet{bills2023language} use an LLM to propose and score explanations of neuron activations at scale. \AlgoName\ differs in objective and output: it uses an LLM agent to propose \emph{task-local predicates} over inputs/outputs that support rule learning, and it attaches those predicates to neurons only after causal validation by interventions.

Recent work studies individual transformer neurons that appear to encode specific facts or relations and uses them as handles for editing model knowledge. Methods like knowledge-neurons \cite{dai2022knowledgeneurons}, TNF-DA \cite{zhou2025editingmemories}, and NSE \cite{jiang2025nse} identify neuron sets that influence a given fact and then update them while trying to preserve overall performance.
\AlgoName\ instead uses neuron activations as objects of explanation for task-defined behaviors: it localizes agonist neurons via ablation and attaches human-readable rules predicting when those neurons control behavior (editing, if any, is downstream). In arithmetic, \citet{nikankin2024arithmetic} also connect localization to semantics, but their semantics are largely hand-assigned: after identifying influential late layers, they map neurons to predefined heuristic categories using Logit Lens \cite{nostalgebraist2020logitlens}, effectively treating neuron effects as shifts in the logit of a target answer token. This framing is most natural when the task can be reduced to a fixed-format prompt and (importantly) a single-token answer. \AlgoName\ removes this constraint by validating neuron influence directly on task behavior, and by learning task-local predicates and rules without assuming a predefined taxonomy of arithmetic heuristics.

Model editing methods aim to change specific behaviors or factual associations without full retraining. Weight-space editors such as ROME, MEMIT, and MEND \cite{meng2022rome,Meng2023MEMIT,Mitchell2022MEND} identify layers and key--value subspaces tied to a fact and apply closed-form or learned updates. More recent work operates at neuron granularity: Neuron Patching \cite{gu2023neuronpatching} attributes failures to individual neurons and adjusts their parameters, while the neuron-level editors discussed above follow a locate-then-edit paradigm. ECE \cite{zhang2025ece} and Model Surgery \cite{wang2025modelsurgery} similarly use neuron attribution to guide parameter changes.
Our algorithm overlaps with these methods only in localization: it identifies behavior-driving neurons via interventions, but does not optimize weights. Its primary output is neuron-anchored rules with task-local causal certificates.
%
A complementary line steers models by manipulating activations rather than weights. Methods such as SADI, LLM-CAS, and T-patcher \cite{Wang2025SADI,zhang2025llmcas,huang2023transformer} inject steering vectors or modules into hidden states, but the resulting directions are typically dense and not conditionally specified by symbolic triggers. \AlgoName\ instead derives sparse if--then rules that specify when intervening on a neuron causally affects the target behavior.

Our work also relates to automated causal circuit discovery. Methods such as ACDC, EAP, and EAP-IG \cite{conmy2023acdc,syed2023attrpatching,hanna2024have} recover sparse subnetworks mediating a behavior via systematic ablations and pruning, yielding graph-structured circuit explanations.
\AlgoName\ focuses on singleton-neuron anchors and uses contrastive hierarchical group ablations to scale localization. This coarse-to-fine intervention strategy is related in spirit to adaptive search in sparse settings (e.g., group-testing style procedures) \cite{aldridge2019group}, but is specialized to causal effects under ablations and to producing neuron-anchored symbolic rules rather than only recovering a sparse set.

\section{Problem Setup} \label{sec:problem_setup}

We formalize three ingredients that \AlgoName{} links: \textit{(i)} a \emph{rule-induced contrastive split} of inputs,
\textit{(ii)} a family of \emph{internal interventions} on the model, and \textit{(iii)} \emph{rule-local causal effect} measures.
The method section (\S\ref{sec:method}) will use these objects to design an efficient localization procedure.

\textbf{Rule-induced contrastive split.}
Let $f_\theta$ be an autoregressive transformer and let $B_{f_\theta}(x)\in\{0,1\}$ be a binary behavior label
for an input $x$ (e.g., correctness under a task metric).
We also assume an interpretable rule $r:\mathcal{X}\to\{0,1\}$ over human-readable predicates of the input.
Intuitively, $r(x)=1$ names a condition under which we hypothesize a particular heuristic is active.

To make post-intervention changes unambiguous, we measure effects on a subset where the \emph{baseline} label is constant.
Fix $b\in\{0,1\}$ and define
$
\mathcal{D}_b \;\coloneqq\; \{x\in\mathcal{D}: B_{f_\theta}(x)=b\}.
$
We refer to $b=1$ as \texttt{baseline-positive} (all baseline outputs succeed) and to $b=0$ as \texttt{baseline-negative}. 
%
Within $\mathcal{D}_b$, the rule partitions inputs into an \emph{\assocText{associated}} slice and an \emph{\unrelText{unrelated}} slice:
$
\Dpos \;\coloneqq\; \{x\in\mathcal{D}_b: r(x)=1\}
$, and
$
\Dneg \;\coloneqq\; \{x\in\mathcal{D}_b: r(x)=0\}
$.
We treat $\Dpm$ as empirical distributions (uniform over elements) unless stated otherwise.
Our goal is to find internal components whose intervention \flipText{changes behavior}, and then to assess whether those changes are selective for $\Dpos$ while \preserveText{preserving} $\Dneg$.

\textbf{Interventions and operational neurons.}
\label{sec:interventions}
We use \emph{neuron} to mean a scalar activation coordinate in the intervention basis chosen for a model component, and \emph{neuron activation} for the scalar value ablated or restored during intervention. This definition is basis-dependent: a neuron may be an MLP-output write coordinate, an attention-head output coordinate, or another component coordinate. Thus, \emph{neuron-anchored} means anchored to such a coordinate.

In this paper, the intervention basis is only the MLP output written into the residual stream. For layer $\ell$ and token position $t$, let $\mathbf{m}_{t}^{\ell}\in\mathbb{R}^{d_{\mathrm{model}}}$ denote that MLP output vector. The atomic neurons in our experiments are the individual coordinates of $\mathbf{m}_{t}^{\ell}$, namely MLP-output channels at the chosen intervention site.

Let $\bar{\mathbf{m}}_t^\ell$ be a baseline vector (e.g., a mean activation).
For a set of coordinates $A\subseteq[d_{\mathrm{model}}]$, define the intervention operator
$
\mathrm{do}(A):\quad m_{t,j}^\ell \leftarrow \bar m_{t,j}^\ell \;\;\text{for all } j\in A
$,
leaving all other activations unchanged. We denote the intervened model by $f_\theta^{\mathrm{do}(A)}$.
Other intervention bases can use the same definitions by replacing $\mathbf{m}_t^\ell$ with the corresponding component activation vector.

\textbf{Rule-local causal effects (flip rates).}
\label{sec:effects}
Let $B_A(x)\coloneqq B_{f_\theta^{\mathrm{do}(A)}}(x)$ denote the post-intervention behavior label.
On each slice, we measure \emph{flips relative to the baseline} as
$\delta_{S\mid b}(A)\coloneqq \mathbb{E}_{x\sim S}[\mathbb{I}\{B_A(x)\neq b\}]$, for $S\in\{\Dpos,\Dneg\}$.
Under \texttt{baseline-positive} ($b=1$), this quantity is the post-intervention \quotes{error} rate; under \texttt{baseline-negative} ($b=0$), it is the post-intervention \quotes{recovery} rate. Thus, the interpretation of the rate depends on the baseline regime. When the regime is fixed by context, we abbreviate $\delta_{+\mid b}(A)$ and $\delta_{-\mid b}(A)$ as $\delta_+(A)$ and $\delta_-(A)$, respectively, where the signs refer to the associated and unrelated slices, not to the baseline label.

\begin{definition}[Strength, selectivity, and agonists]
\label{def:strength_selectivity}
Using the flip rates from \S\ref{sec:effects}, $b=1$ measures the $(1{\to}0)$ direction and $b=0$ measures the $(0{\to}1)$ direction.
For a neuron set $A$, define regime-conditional strength as $E_b(A)\coloneqq \max\{\delta_{+\mid b}(A),\delta_{-\mid b}(A)\}$, absolute slice selectivity as $\mathrm{Sel}_b(A)\coloneqq |\delta_{+\mid b}(A)-\delta_{-\mid b}(A)|$, and signed selectivity as $\mathrm{Sel}_{\mathrm{sgn},b}(A)\coloneqq \delta_{+\mid b}(A)-\delta_{-\mid b}(A)$.
Thus, $E_b$ measures strength, $\mathrm{Sel}_b$ unsigned slice specificity, and the sign of $\mathrm{Sel}_{\mathrm{sgn},b}$ indicates whether $\Dpos$ or $\Dneg$ flips more.
A singleton $j$ is a \emph{$\tau$-agonist} in regime $b$ if $E_b(\{j\})\ge\tau$; it is \emph{$\Dpos$-selective} if $\mathrm{Sel}_{\mathrm{sgn},b}(\{j\})\ge\varepsilon$, \emph{$\Dneg$-selective} if $\mathrm{Sel}_{\mathrm{sgn},b}(\{j\})\le-\varepsilon$, and \emph{non-selective} if $\mathrm{Sel}_b(\{j\})<\varepsilon$.
When a single regime is fixed, we suppress $b$ and write $\delta_S(A)$, $E(A)$, $\mathrm{Sel}(A)$, and $\mathrm{Sel}_{\mathrm{sgn}}(A)$.
\end{definition}

\textbf{Interpretation.}
Rule-induced slices define the causal search regime. Stage~3 localizes agonists using UCBs on $E_b$; later analyses evaluate selectivity and anchored-rule quality. Group-level non-additivities motivate the search procedure developed below.

\section{\AlgoName: Rule-Grounded Neuron Localization}
\label{sec:method}

\begin{figure*}
\centering
\includegraphics[width=.8\linewidth]{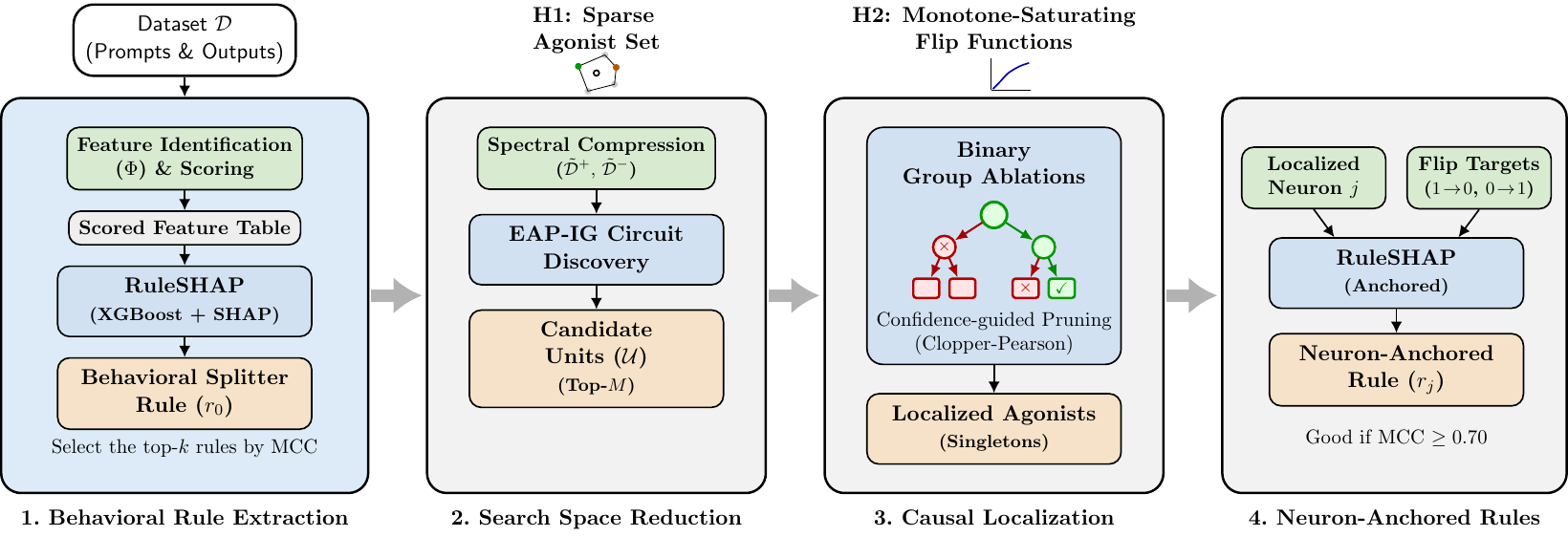}
\Description{Flowchart of the \AlgoName pipeline in four stages. (1) Behavioral rule extraction: RuleSHAP (XGBoost + SHAP) takes dataset $\mathcal{D}$ (prompts and outputs), scores clauses by Matthews correlation coefficient (MCC), and forms a greedily compactified OR-combination $r_0$ that keeps only clauses improving the rulebase-level MCC. (2) Search space reduction: spectral compression and edge attribution patching with integrated gradients (EAP-IG) yield reduced datasets $\tilde{\mathcal{D}}^+$ and $\tilde{\mathcal{D}}^-$ and per-layer candidate sets $U^\ell$ (top-$n$). (3) Causal localization: binary hierarchical group ablations with confidence-guided UCB filtering (Clopper--Pearson) localize UCB-thresholded candidate agonist singletons within each layer, with flip directions $(1{\to}0)$ and $(0{\to}1)$. (4) Neuron-anchored rules: anchored RuleSHAP produces a greedily compactified OR-combination $r_j$, reported as HQ-T when held-out MCC $\ge 0.70$ and as HQ-F when full-fit MCC $\ge 0.70$, with a minimum dataset-coverage floor. Two hypotheses are shown: H1 sparse agonist set and H2 regime-conditional monotone-saturating flip functions.}
\caption{Pipeline overview: RuleSHAP extracts greedily compactified behavioral splitter rules from $\mathcal{D}$, spectral compression and EAP-IG reduce the search space, binary hierarchical group ablations localize UCB-thresholded candidate agonists, and RuleSHAP yields greedily compactified neuron-anchored rules.}
\label{fig:pipeline_overview}
\end{figure*}

\AlgoName{} is instantiated as a four-stage pipeline (shown in Fig.~\ref{fig:pipeline_overview}): extract a behavioral splitter $\rzero$, reduce the evaluation space and candidate neuron set $U$, localize high-strength candidate agonists with confidence-guided hierarchical ablations, and learn neuron-anchored rules $\ranch$ that predict when localized agonists flip behavior.

\subsection{Design hypotheses for efficient localization} \label{sec:method:hypotheses}

\textbf{Hypothesis 1: sparse gating under a bag-of-heuristics view.}
We treat the rule-induced contrast $(\Dpos,\Dneg)$ as an attempt to isolate a \emph{single} behavioral mechanism.
If $r$ is a faithful trigger for a heuristic, then conditioning on $r(x)=1$ should concentrate the success/failure difference onto a small set of internal components (i.e., neuron-level activation patterns) that are \emph{present in} $\Dpos$ and largely \emph{absent from} $\Dneg$.
At MLP write sites, channels can act as feature-selective units, and prior work on structured behaviors such as arithmetic reports that performance can be mediated by small subsets of MLP neurons that fire on specific input patterns and shift answer logits \cite{nikankin2024arithmetic}. The localization and editing literature reviewed in \S\ref{sec:related_work} likewise suggests that task- or fact-specific effects can sometimes be carried by sparse neuron subsets rather than by diffuse changes over all channels.

In our setting, baseline filtering (Definition~\ref{def:strength_selectivity}) reduces ambiguity: on $\mathcal{D}_b$ the only way an intervention changes $B_{f_\theta}$ is by altering the internal evidence needed to preserve the baseline outcome.
If $r$ genuinely names when the heuristic is active, then the causal neurons that implement that heuristic should \textit{(i)} have measurable strength on $\Dpos$ and \textit{(ii)} exhibit positive signed selectivity $\mathrm{Sel}_{\mathrm{sgn}}>0$, because they are recruited mainly in the rule-\assocText{associated} slice.

Conversely, arbitrary or intentionally incorrect partitions need not isolate a reused mechanism, so Stage~3 may find no small set that preferentially \flipText{flips} $\Dpos$ while \preserveText{preserving} $\Dneg$ (Definition~\ref{def:strength_selectivity}). Spectral-only splits (no rule) are not a negative control: they lack symbolic semantics but can work when representation geometry aligns with behavior, so we use them as a rule-free fallback.
Notably, the search space reduction (Stage~2) only improves efficiency by restricting candidates; it cannot induce selectivity when the split itself does not correspond to a shared causal trigger.

\textbf{Hypothesis 2: regime-conditional dominance, overlap, and fast localization.}
Although our interventions are additive at the activation site (\S\ref{sec:interventions}), end-to-end behavioral effects need not be additive, because the residual stream is processed through attention softmax, LayerNorm, and MLP nonlinearities \citep{vaswani2017attention,ba2016layernorm,hendrycks2016gelu}.
Empirically (cf. \S\ref{sec:results}), group ablations can exhibit two distinct effects within a fixed baseline regime $b$: \emph{dominance}, where one or a few neurons explain most of the group-level change, and \emph{overlap}, where influential neurons affect many of the same examples, causing saturation.

For instance, on arithmetic prompts initially answered correctly, replacing a group of MLP-output channels may turn many answers incorrect. If replacing a single channel already flips most of those same prompts, while the other channels add little or no new coverage, the group effect should not be interpreted as many independent small causes. In that regime, the effect is mostly saturated by a dominant channel. We call this the \emph{overtopping} pattern. Conversely, if a singleton effect is hidden at the parent-group scale, CHA has no reason to descend into that branch, so brute-force comparisons may reveal missed agonists in such non-overtopping regimes.

For a fixed regime $b$ (direction $b{\to}1{-}b$), CHA is most reliable when slice-wise flip rates behave like monotone, saturating set functions. One concrete abstraction is the regime-conditional union-of-flips model $\delta_{S\mid b}(A)=\Pr_{x\sim S}[F^{(b)}_{S,A}(x)]$, where $F^{(b)}_{S,A}(x)$ denotes that at least one neuron $j\in A$ flips $x$ on slice $S$. This implies monotonicity and is compatible with coverage-style submodularity under mild assumptions \citep{nemhauser1978analysis,fujishige2005submodular}.
The same dominance or saturation pattern need not hold in the dual regime $1-b$; it depends on the task, model, intervention, baseline-filtered distribution, and splitter-induced $+$/$-$ slices. 
Operationally, \AlgoName{} exploits this overtopping structure to make rule-guided localization efficient: CHA uses coarse group effects to find dominant rule-relevant singletons, which then become candidates for symbolic anchoring.

\subsection{Stage 1: Behavioral Rule Extraction}

\textbf{Feature identification and scoring.}
\label{sec:method:features}
We define a \emph{behavioral splitter} $\rzero(x)$ as a \ruleText{Boolean rule} over an interpretable predicate vocabulary $\Phi$ that predicts task success on a corpus $\mathcal{S}$ of records $x$ (prompt plus model output and any parsed fields).
We construct $\Phi=\{\varphi_1,\dots,\varphi_p\}$ in two steps:
\textit{(i)} a small set of task-specific \emph{seed} predicates, and
\textit{(ii)} a contrastive, LLM-assisted predicate-proposal loop in which an LLM inspects batches of successful vs.\ failed examples and proposes additional deterministic predicates intended to discriminate success from failure.
Each candidate predicate is evaluated on $\mathcal{S}$ and scored against the task label using \textit{(i)} AUC, \textit{(ii)} average precision above base rate, and \textit{(iii)} standardized mean separation $\Delta/\sigma$.
We optionally remove low-signal predicates, near-duplicates via predicate--predicate correlation, and extreme outliers via a median-absolute-deviation (MAD) heuristic.
Operational details for predicates proposal, scoring, and filtering are in Appendix~\ref{app:operational_settings}.

\textbf{Behavioral rule extraction with RuleSHAP.}
\label{sec:method:rules}
Given the scored predicate vocabulary, we extract if--then rules using RuleSHAP~\cite{sovrano2026ruleshap}, without LASSO-based pruning.
RuleSHAP fits gradient-boosted trees (XGBoost~\cite{chen2016xgboost}) on the rule-construction split, enumerates candidate clauses, and uses SHAP weights~\cite{lundberg2017unified} to rank clauses by contribution to the task label.
We then select a compact OR-combination
$\rzero(x)=\bigvee_{i\in \mathcal{I}} r_i(x)$
by seeding with clauses ranked by rule-construction MCC~\cite{chicco2020advantages} and RuleSHAP importance, then greedily adding clauses when the OR-combination improves MCC.
The selected clauses ${r_i : i\in\mathcal{I}}$ are carried forward as behavioral splitter rules: downstream sampling, circuit discovery, and CHA treat each clause as a rule-induced contrast.
Stage~4 runs a separate RuleSHAP pass for each localized neuron, using flip targets rather than the original task-success target, to produce neuron-anchored rules.

For evaluation, we create a train/test split stratified by clusters in the predicate space $\Phi$, keeping held-out data points diverse with respect to $\Phi$.
The behavioral splitter $\rzero$ is constructed only on the train split; the test split is not used to select $\rzero$ or localize neurons, and is reserved for Stage~4 scoring of neuron-anchored rules (\S\ref{sec:method:anchoring}).

\subsection{Stage 2: Search Space Reduction}
\label{sec:method:candidates}

Mechanistic evaluation is the bottleneck, so we reduce both \textit{(i)} the evaluation workload and \textit{(ii)} the neuron search space.
We first compress the rule-induced slices, then use fast attribution to produce a candidate set $U$ for exact ablations.

\textbf{Spectral data compression.} \label{sec:method:spectral}
For each record $x$, we compute a vector representation (the final normalized hidden state pooled at the last prompt token in the main runs), project all points into a low-dimensional PCA space, and select representative points per slice with a farthest-first $k$-center-style coverage heuristic~\cite{gonzalez1985clustering}.
To reduce distribution shift between slices, selected \assocText{associated} examples are paired with \unrelText{unrelated} ones. This produces paired datasets $(\tDpos,\tDneg)$ used for attribution and ablation.
Unless stated otherwise, Stages~2--3 estimate attribution and ablation on $(\tDpos,\tDneg)$.

\textbf{Candidate reduction via EAP-IG circuit discovery.} \label{sec:method:eap}
To avoid exhaustive singleton ablations, we use EAP-IG as an upstream circuit-scoring heuristic and then run interventions only on the retained scalar coordinates. This first part follows the standard pattern in automated circuit discovery \citep{syed2023attrpatching,hanna2024have,conmy2023acdc}, but our claims do not depend on EAP-IG being the best circuit-discovery method: ACDC or another method could be used instead, and all recall numbers below are conditional on the retained candidate set.
For each splitter $\rzero$, we run EAP-IG on clean/corrupted associated/unrelated pairs induced by the compact slices $(\tDpos,\tDneg)$.
EAP-IG linearly interpolates between activations induced by the paired slices and accumulates gradients along the path, yielding scores for heterogeneous circuit elements. For the experiments in this paper, exact ablations are restricted to MLP-output write coordinates; non-MLP elements such as attention heads are filtered out before CHA. In the Qwen2-1.5B arithmetic runs, the default EAP-IG export cap $M=100{,}000$ leaves $43{,}008$ MLP-write coordinates with non-zero EAP-IG scores.
We group the retained coordinates by layer into $U^\ell \!\subseteq\! [d_{\mathrm{model}}]$ for ablations. 

\subsection{Stage 3: Causal Localization}

Given $(\tDpos,\tDneg)$ induced by $\rzero$ and per-layer candidates $U^\ell$, contrastive hierarchical ablation (CHA) searches each layer independently with UCB-based conservative pruning: it descends into a group when its UCB strength remains at least $\tau$, prunes groups whose UCB strength falls below $\tau$, and retains singleton leaves only if their UCB strength is at least $\tau$. Rule selectivity $\mathrm{Sel}(\cdot)$ is recorded after localization for reporting. Interventions use $\mathrm{do}(\cdot)$ from \S\ref{sec:interventions} on the MLP-write vector $\mathbf{m}_t^\ell$, with mean-positional replacement in the main runs and alternative baselines in Appendix~\ref{app:operational_settings}.

\textbf{Outcome, baseline label, and flip rates.}
Stage~3 estimates the flip rates from \S\ref{sec:effects} on the compact evaluation slices $(\tDpos,\tDneg)$. Under an intervention $\mathrm{do}(A)$, let $C_S(A,x)\in\{0,1\}$ denote the post-intervention outcome on slice $S\in\{+,-\}$. The empirical compact-slice estimate is
$\widehat{\delta}_{S\mid b}(A)=\Pr_{x\sim \tilde{\mathcal{D}}^S}\big[C_S(A,x)\neq b\big]$.
When the regime is fixed, we write $\widehat{\delta}_S(A)$ for the corresponding compact-slice estimate.
For reporting, we also use compact-slice post-intervention accuracy, $a_S(A)=\Pr_{x\sim \tilde{\mathcal{D}}^S}\big[C_S(A,x)=1\big]$, which satisfies: if $b=1$ (\texttt{baseline-positive}), then $\widehat{\delta}_{S\mid b}(A)=1-a_S(A)$; instead, if $b=0$ (\texttt{baseline-negative}), then $\widehat{\delta}_{S\mid b}(A)=a_S(A)$.


\textbf{Group-testing view and binary hierarchical ablations.}
Fix a baseline regime $b$ and strength threshold $\tau$, and define the regime-conditional strength predicate
$
T_{\tau,b}(A)\;=\;\mathbb{I}\{E_b(A)\ge\tau\}.
$
Under Hypothesis~1 plus a separation condition (groups without $\tau$-agonists test below $\tau$), $T_{\tau,b}(\cdot)$ behaves like a group-testing oracle: a group tests positive exactly when it contains at least one $\tau$-agonist.
In that idealized separated case, a coarse-to-fine binary hierarchical search identifies all $\tau$-agonists with $O\!\left(k\log\frac{N}{k}\right)+O(k)$
adaptive group ablations for $N$ candidates and $k\ll N$ strong neurons \citep{hwang1972method,du2000grouptesting}.

Without separation, monotonicity still supports safe pruning of certified sub-threshold groups, but it does not by itself imply the group-testing complexity bound. Hypothesis~2 motivates the search when effects saturate and a target singleton remains detectable at parent-group scale. If that property fails, CHA may miss a singleton that brute-force would find.

\textbf{Confidence-guided pruning (finite-sample).}
For the fixed regime under evaluation, we compute $\widehat{\delta}_\pm(A)$ on mini-batches from $(\tDpos,\tDneg)$ and use binomial intervals (Clopper--Pearson) to form one-sided upper UCBs on compact-slice strength \citep{clopper1934use}.
Let $u_+$ and $u_-$ be one-sided Clopper--Pearson UCBs for $\delta_+(A)$ and $\delta_-(A)$, each at level $\alpha/2$. We set $\widehat U_E(A)=\max(u_+,u_-)$ and use this statistic for conservative pruning and singleton inclusion. 

In the reported experiments, Algorithm~\ref{alg:cha} uses the same fixed per-group confidence parameter at every queried node on compact \assocText{associated}/\unrelText{unrelated} evaluation subsets reused within a circuit, rather than a fully global alpha-spending schedule over the whole adaptive tree. A much more conservative global allocation was substantially less effective in our finite-sample regime because node-level upper bounds became too loose to certify many groups as sub-threshold. A proof of the stronger idealized pruning-safety theorem is in Appendix~\ref{app:cha_validity}.

\begin{algorithm}[t]
\small
\caption{Contrastive Hierarchical Ablation (CHA), single-layer.} \label{alg:cha}
\begin{algorithmic}[1]
\REQUIRE Candidate set $U^\ell\subseteq[d_{\mathrm{model}}]$, strength threshold $\tau$, selectivity threshold $\varepsilon$ for reporting, per-node confidence parameter $\alpha$, compact evaluation subset from $(\tDpos,\tDneg)$
\ENSURE Localized UCB-thresholded candidate agonists $\mathcal{J}$; selectivity summaries for optional reporting
\STATE \textbf{function} Search$(U,\alpha)$
  \STATE Estimate $\widehat{\delta}_\pm(U)$ on the current compact evaluation subset; compute $\widehat U_E(U;\alpha)$
  \IF{$\widehat U_E(U;\alpha) < \tau$}
     \STATE \textbf{return} $\emptyset$ \COMMENT{Conservative prune}
  \ENDIF
  \IF{$|U|=1$}
   \STATE Let $U=\{j\}$
   \STATE Estimate $\widehat{\delta}_\pm(\{j\})$ on the current compact evaluation subset
   \STATE Compute $\widehat U_E(\{j\};\alpha)$ and
          $\widehat{\mathrm{Sel}}(\{j\})=\left|\widehat{\delta}_+(\{j\})-\widehat{\delta}_-(\{j\})\right|$
   \STATE \textbf{return} $\{j\}$ if $\widehat U_E(\{j\};\alpha)\ge\tau$, else $\emptyset$ \COMMENT{No selectivity-based pruning}
  \ENDIF
  \STATE Split $U$ into two children $(U_1,U_2)$
  \STATE \hspace{\algorithmicindent}\emph{Implementation:} root split uses the sign of the EAP-IG attribution score; later splits use balanced bisection
  \STATE \textbf{return} $\mathrm{Search}(U_1,\alpha)\ \cup\ \mathrm{Search}(U_2,\alpha)$
\STATE \textbf{end function}
\STATE
\STATE $\mathcal{J} \gets \mathrm{Search}(U^\ell,\alpha)$
\end{algorithmic}
\end{algorithm}

\subsection{Stage 4: Neuron-Anchored Rules}
\label{sec:method:anchoring}

For each UCB-localized singleton candidate agonist $j \!\in\! \mathcal{J}$, we learn a \ruleText{symbolic rule} $\ranch(x)$ over $\Phi$ that predicts \emph{when} ablating $j$ flips the baseline behavior.
Let $C(A,x)\in\{0,1\}$ denote the post-intervention outcome for example $x$ under $\mathrm{do}(A)$ (dropping the slice index for brevity).
For each example $x\in\mathcal{D}_b$ and localized candidate $j$, we define the flip target $y^{(j)}(x) \!=\! \mathbb{I}\{C(\{j\},x)\neq b\}$, i.e., whether ablating $j$ changes the baseline outcome on $x$.
When $b=1$ this corresponds to the $(1{\to}0)$ direction; when $b=0$ it corresponds to $(0{\to}1)$.
We then run RuleSHAP again to emit \ruleText{Boolean rule} candidates $\ranch(x)$ that predict the relevant flip target; the scoring scope is defined below.

A localized candidate agonist is \emph{rule-bearing under $\Phi$} if Stage~4 extracts a neuron-anchored rule for it. As for $\rzero$, OR-combinations are greedily compactified: a clause is kept only if it improves MCC for the whole neuron-specific rulebase. We require a rule's dataset coverage, the fraction of eligible examples on which the rule fires, to be at least $0.5\%$ before the rule can contribute to high-quality counts. At $\mathrm{MCC}\ge0.70$, \emph{HQ-test} (HQ-T) emits clauses on the training split, compactifies them, and scores the OR-combination on the held-out split; \emph{HQ-full} (HQ-F) compactifies and scores the OR-combination on all observed eligible data points. HQ-T is therefore a held-out support check, whereas HQ-F measures descriptive all-fit structure under $\Phi$.
Anchored rules can \textit{(i)} refine $\rzero$ by identifying a sub-clause that more tightly matches the causal trigger, or \textit{(ii)} expose correlations that the splitter $\rzero$ entangles. When an anchored rule uses only prompt-observable features, it can also serve as a one-pass runtime gate. 

\section{Experimental Setup}
\label{sec:exp_setup}

The mechanistic goal is \emph{singleton} causal localization: finding activations whose ablation induces a strong flip pattern, then assessing whether the pattern is rule-selective (Definition~\ref{def:strength_selectivity}). Stage~3 therefore filters by strength; Stage~4 anchors only localized singletons whose flip targets are reliably predictable from $\Phi$. The rest remain localized causal candidates without anchored rules.

\textbf{Models.}
We evaluate \AlgoName on open-weight LLMs that support activation-level interventions in TransformerLens \cite{nanda2022transformerlens}: \texttt{Qwen2\ttdash{}7B\ttdash{}Instruct} and \texttt{Qwen2\ttdash{}1.5B\ttdash{}Instruct}~\cite{yang2024qwen2}, and \texttt{GPT\ttdash{}J\ttdash{}6B}~\cite{wang2021gptj}.
The 7B Qwen2 model is the primary Qwen2 setting; the 1.5B Qwen2 model is included to test whether the same pipeline behavior appears at a smaller scale.
We intervene on the MLP-write coordinates $\mathbf{m}_t^\ell$ defined in \S\ref{sec:interventions}, replacing selected channels with a baseline value.

\textbf{Tasks and datasets.}
The main study uses arithmetic and Best-of-$N$ jailbreaking as task behaviors: exact-answer correctness and harmful-completion success. Appendix~\ref{app:additional_results} adds HANS NLI \cite{mccoy2019right} to test whether the pipeline finds anchored rules beyond these tasks.

\emph{Arithmetic.} Following \cite{nikankin2024arithmetic}, prompts have the form \texttt{a <op> b =} with greedy decoding (max 6 new tokens).
For $+$, $-$, and $\times$, we enumerate all $(a,b)\in[0,299]^2$; for division, we enumerate $(a,b)\in[1,299]^2$ to avoid division by zero and score the parsed numeric response against the exact quotient, accepting decimal truncations.

\emph{Jailbreaking.} We use the BoN jailbreaking benchmark of \citet{hughes2024best} (harmful behaviors paired with augmented attack prompts). Each augmented prompt is one evaluated attack candidate: we generate one greedy completion per prompt and label it with the benchmark's automated LLM-based safety grader, which marks substantive unsafe assistance as harmful and refusals or benign safety information as safe. Here $B(x)=1$ denotes jailbreaking success (a harmful completion), not a desirable outcome, and $B(x)=0$ denotes a safe outcome/refusal. Thus, $(1{\to}0)$ is jailbreak-suppressing, whereas $(0{\to}1)$ is jailbreak-inducing and is reported as a risk signal rather than as an improvement.

\textbf{Pipeline instantiation (Stages 1--4).}
Unless stated otherwise:
\textit{(i)} the predicate vocabulary $\Phi$ is the union of task-specific seed predicates and LLM-proposed predicates,
\textit{(ii)} RuleSHAP is used as a high-recall rule generator with LASSO disabled (\S\ref{sec:method:rules}),
\textit{(iii)} we analyze the top $K$ highest-quality splitter rules per model (by rule-selection MCC), using $K=5$ for arithmetic and $K=1$ for the others,
\textit{(iv)} circuit discovery uses EAP-IG on matched \assocText{associated}/\unrelText{unrelated} pairs (\S\ref{sec:method:eap}) and exports an upstream top-$M$ EAP-IG list; exact hierarchical ablations are then run on retained MLP-output write coordinates using compact baseline-specific slices after filtering non-MLP elements.

\textbf{Evaluation metrics (rules and overlaps).}
Splitter rules and neuron-anchored rules are binary predictors, scored by MCC \cite{chicco2020advantages} because it remains informative under class imbalance. As in \S\ref{sec:method:anchoring}, HQ-F measures observed-set fit under $\Phi$, while HQ-T is the held-out (test-set-only) check; both use $\mathrm{MCC}\ge0.70$ after the dataset-coverage floor. To quantify overlap between two localized-agonist sets $A$ and $J$ (e.g., across tasks), we report the Jaccard index $|A\cap J|/|A\cup J|$. We also use \emph{localized-agonist union coverage}, \(\mathrm{Cov}_{\cup}(\mathcal{J}; b{\to}1{-}b)\), for the fraction of examples in regime $b$ flipped by any singleton in $\mathcal{J}$, estimated empirically as $\Pr_{x\sim\mathcal{D}_b}[\exists j\in\mathcal{J}: B_{\{j\}}(x)\neq b]$.

\textbf{Rule-local intervention protocol.}
All mechanistic measurements follow the baseline-filtered setup from \S\ref{sec:problem_setup}.
For each splitter rule $\rzero$, we evaluate both \texttt{baseline-positive} ($b=1$) and \texttt{baseline-negative} ($b=0$) regimes, inducing slices $(\Dpos,\Dneg)$ within each regime. These regimes are analyzed separately; observed dominance, overlap, and pruning behavior in one flip direction are not treated as evidence for the dual direction.
We report compact-slice flip-rate estimates $\widehat{\delta}_+(A),\widehat{\delta}_-(A)$, together with the corresponding strength and selectivity summaries (Definition~\ref{def:strength_selectivity}).
Unless stated otherwise, ablations replace targeted coordinates with a \texttt{mean-positional} baseline, i.e., a coordinatewise mean activation estimated for the corresponding layer and generation position from 256 prompts. Appendix~\ref{app:donor_baseline} also defines donor-grounded variants; in particular, mean-donor replaces each coordinate with the observed activation value closest to that coordinate's empirical mean.
For efficiency and to avoid perturbing prompt encoding, we use a \emph{post-prefill, decode-time intervention protocol}: the prompt is processed normally, and interventions are applied only during autoregressive output-token generation. Thus, our interventions test causal influence on generation given a fixed prompt representation, rather than causal effects on prompt encoding.
Unless otherwise specified, each group evaluation uses up to 64 examples per slice, and CHA uses thresholds $\tau=0.2$ (strength) and $\varepsilon=0.2$ (rule selectivity) with pruning confidence level $\alpha=0.05$.

\textbf{Experiments (what validates H1/H2 and the full pipeline).}

\emph{E0: pipeline validation.}
For each model/task, we run Stages~1--4 and report localized agonist counts, HQ-T and HQ-F anchored-rule counts, and localized-agonist union coverage.

\emph{E1: split-and-coverage conditions and H1.}
To test whether sparse, rule-selective singletons are helped by rule-based splits and spectral-coverage evaluation sets, we compare four named conditions:
\textit{(a)} \textit{Rule split + spectral coverage}, induced by learned splitters $\rzero$ (main condition),
\textit{(b)} \textit{Spectral split (no rule)}, where the \assocText{associated}/\unrelText{unrelated} slices are formed by spectral clustering in representation space without a learned rule,
\textit{(c)} \textit{Rule split + random coverage}, which keeps the learned splitter but samples compact evaluation points uniformly, and
\textit{(d)} \textit{Fake-rule split + spectral coverage}, where targets are permuted before rule extraction to produce intentionally mismatched splitters.
We compare discovered singleton agonists, rule-bearing agonist counts across MCC thresholds, and localized-agonist union coverage. H1 predicts the clearest degradation for \textit{(d)}, while \textit{(b)} serves as a strong rule-free fallback rather than as a guaranteed negative control.

\emph{E2: CHA vs.\ exhaustive singletons + H2.}
To evaluate regime-aware localization, we include an ablation baseline (the \textit{brute-force} baseline) that exhaustively evaluates \emph{all singleton ablations} over the same retained MLP-write coordinates as CHA.
The brute-force baseline is computationally expensive and is completed only for one splitter circuit per task--model--baseline regime.
For a valid comparison, each CHA result in E2 is therefore filtered to the same circuit identifier used by the corresponding brute-force run; the E2 numbers do not compare brute-force on one circuit against CHA aggregated over all circuits.
We then compare:
\textit{(i)} the number of singleton $\tau$-agonists found by CHA vs.\ brute-force on that shared circuit,
\textit{(ii)} overlap/recall of the CHA-localized candidate agonists under brute-force ground truth, and
\textit{(iii)} the intervention cost, measured as \% of ablations performed by CHA.
This tests whether coarse-to-fine search remains accurate under dominance/overlap effects, and whether failures concentrate in regimes where target singletons are weak or not detectable at parent-group scale (H2).

\section{Results} \label{sec:results}

\begin{table}[t]
\centering
\resizebox{\columnwidth}{!}{%
\begin{tabular}{llrrrrrr}
\toprule
Task & Model & \multicolumn{1}{c}{\begin{tabular}[c]{@{}c@{}}Task\\score\end{tabular}} & \multicolumn{1}{c}{\begin{tabular}[c]{@{}c@{}}Loc.\\$\tau$-ag.\end{tabular}} & HQ-T & HQ-F & \multicolumn{1}{c}{\begin{tabular}[c]{@{}c@{}}Removed\\$1{\to}0$\\elig. (\%)\end{tabular}} & \multicolumn{1}{c}{\begin{tabular}[c]{@{}c@{}}Induced\\$0{\to}1$\\elig. (\%)\end{tabular}} \\
\midrule
Arithmetic & Qwen2-7B & 71.8 & 135 & 114 & 41 & 100.0 & 72.8 \\
 & Qwen2-1.5B & 55.2 & 40 & 12 & 10 & 100.0 & 48.6 \\
 & GPT-J-6B & 2.0 & 310 & 115 & 224 & 100.0 & 11.7 \\
Jailbreaking & Qwen2-7B & 9.6 & 120 & 49 & 120 & 97.6 & 28.7 \\
 & Qwen2-1.5B & 6.6 & 199 & 81 & 199 & 100.0 & 32.5 \\
 & GPT-J-6B & 2.7 & 945 & 28 & 616 & 100.0 & 17.3 \\
\bottomrule
\end{tabular}}
\caption{Baseline task score (arithmetic accuracy or jailbreak rate), localized $\tau$-agonist counts, HQ-T/HQ-F counts after the dataset-coverage floor, and eligible-regime union flip rates for all localized agonists.}
\label{tab:results:e0}
\end{table}

\textbf{E0: \AlgoName separates causal localization from rule quality.}
Table~\ref{tab:results:e0} reports unablated task performance, localized agonist counts, HQ-T/HQ-F counts, and eligible-regime union coverage. Localized agonists cover 97.6--100.0\% of eligible correct arithmetic answers and successful jailbreaks, while reverse-direction coverage reaches 72.8\% for arithmetic errors and 32.5\% for safe responses. 
At MCC $\ge0.70$, 399 anchors pass HQ-T across these six runs (arithmetic $114/12/115$ for Qwen2-7B, Qwen2-1.5B, and GPT-J; jailbreaking $49/81/28$ in the same model order). HQ-F totals 1210, reflecting all-fit symbolic predictability on the eligible data points. 
In the $(0{\to}1)$ direction, coverage is also strongly correlated with unablated task score (Spearman $\rho=0.943$, $n=6$).
Figure~\ref{fig:e0:layer_coverage_arith_qwen2} further shows that, for arithmetic on Qwen2-7B-Instruct, localized agonists and HQ-F anchors concentrate mainly in early MLP-write layers.
This complements attribution-based rankings: Appendix~\ref{app:eap_vs_cha_qwen2_arith} shows that EAP-IG concentrates top coordinates in late layers, while CHA's strongest localized agonists concentrate in early layers.
Because exhaustive singleton ablation is expensive, Experiment~E2 reports only circuits where brute-force was feasible, comparing brute-force and CHA on the same MLP-write candidate set.

\begin{figure}[t]
\centering
\includegraphics[width=\linewidth]{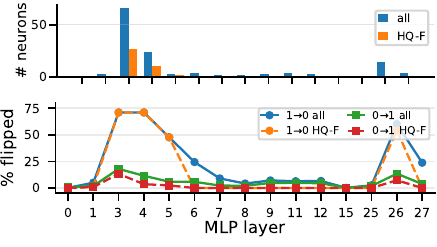}
\Description{Two-panel plot for arithmetic on Qwen2-7B-Instruct. Top: per-layer counts of rule-bearing agonists and HQ-F agonists with anchored-rule MCC \ge 0.70. Bottom: per-layer unique flip rates for correct-to-incorrect and incorrect-to-correct flips, reported for all rule-bearing agonists and for HQ-F agonists.}
\caption{Arithmetic/Qwen2-7B-Instruct: layerwise counts and unique flip rates for rule-bearing agonists and HQ-F agonists under \textit{Rule split + spectral coverage}.}
\label{fig:e0:layer_coverage_arith_qwen2}
\end{figure}

\textbf{E1: Rule-aligned splits and spectral coverage are useful.}
Table~\ref{tab:results:e1} sweeps both HQ-T and HQ-F thresholds over rule-bearing agonists across arithmetic and jailbreaking task--model runs. Rule split + spectral coverage gives the highest aggregate total at every threshold in both scopes. At the main cutoff $t \!=\! 0.70$, it yields $399$ HQ-T anchors, compared with $342$ for spectral-only splits, $258$ for rule split + random coverage, and $124$ for fake-rule splits; in HQ-F, the corresponding counts are $1210$, $1087$, $866$, and $142$. Even at $t \!=\!0.85$, the same ordering holds for HQ-T ($58$, $54$, $41$, $29$) and HQ-F ($949$, $859$, $797$, $97$). 

\begin{table}[t]
\centering
\setlength{\tabcolsep}{3pt} 
\resizebox{\linewidth}{!}{%
\begin{tabular}{lrrrrrr}
\toprule
Method & Rule-bearing & $t=.70$ & $t=.75$ & $t=.80$ & $t=.85$ & $t=.90$ \\
\midrule
\textbf{Rule split + spectr. cov.} & \textbf{1486/1343} & \textbf{399/1210} & \textbf{271/1173} & \textbf{151/1073} & \textbf{58/949} & \textbf{12/926} \\
Spectr. split (no rule) & 1211/1154 & 342/1087 & 233/1061 & 135/970 & 54/859 & 10/840 \\
Rule split + random cov. & 1170/979 & 258/866 & 178/850 & 102/821 & 41/797 & 6/784 \\
Fake-rule split + spectr. cov. & 242/196 & 124/142 & 91/131 & 63/113 & 29/97 & 5/94 \\
\bottomrule
\end{tabular}
}
\caption{Overall rule-bearing localized agonists. Each entry is HQ-T/HQ-F; threshold columns count agonists with MCC $\ge t$ after the dataset-coverage floor.}
\label{tab:results:e1}
\end{table}

\textbf{E2: CHA works well in the high-effect regimes it is designed to target.}
Completed brute-force comparisons evaluate CHA against singleton ablation over the same retained MLP-write candidate set and the same splitter circuit.
This measures CHA search recall conditional on both the candidate reducer and the single splitter circuit for which brute-force was completed.
Table~\ref{tab:results:e2} reports the completed baseline-regime comparisons for Qwen2-family and GPT-J models.
CHA recovers $791/1198$ brute-force singleton agonists overall ($66.0\%$) and $129/133$ in the strongest singleton-effect tier $E_b(\{j\})\in[0.5,1.0]$ ($97.0\%$). Recall is higher in the strongest tier than in the weakest tier under a one-sided paired $t$-test on baseline-specific task--model comparisons with both bins present ($t_7{=}2.95$, $p{=}0.011$). Positive-baseline comparisons also have higher overall recall than their matched negative-baseline comparisons across the task--model pairs in this compact comparison ($t_4{=}4.76$, $p{=}0.0045$).

\begin{table}[t]
\centering
\resizebox{\columnwidth}{!}{%
\begin{tabular}{llrrrrr}
\toprule
Task/model & $b$ & CHA cost & Overall & $[0.2,0.3)$ & $[0.3,0.5)$ & $[0.5,1.0]$ \\
\midrule
Arith./Qwen2-7B & 1 & $1.96\%$ & 120/160 & 51/87 & 21/25 & 48/48 \\
 & 0 & $1.46\%$ & 11/168 & 6/117 & 4/48 & 1/3 \\
Arith./Qwen2-1.5B & 1 & $2.63\%$ & 38/38 & 14/14 & 7/7 & 17/17 \\
 & 0 & $0.42\%$ & 3/23 & 3/20 & 0/2 & 0/1 \\
Arith./GPT-J & 1 & $4.06\%$ & 298/398 & 168/263 & 106/111 & 24/24 \\
 & 0 & $0.41\%$ & 10/16 & 1/4 & 2/4 & 7/8 \\

Jail./Qwen2-7B & 1 & $2.81\%$ & 120/140 & 74/94 & 36/36 & 10/10 \\
 & 0 & $0.04\%$ & 0/8 & 0/8 & 0/0 & 0/0 \\
Jail./Qwen2-1.5B & 1 & $10.27\%$ & 190/236 & 129/174 & 39/40 & 22/22 \\
 & 0 & $0.14\%$ & 1/11 & 1/10 & 0/1 & 0/0 \\
All completed & all & $2.14\%$ & 791/1198 & 447/791 & 215/274 & 129/133 \\
\bottomrule
\end{tabular}}
\caption{E2: CHA vs.\ brute-force. We report overall recall and CHA cost as the \% of ablations performed by CHA for the same splitter circuit. Bracketed columns group CHA-localized vs. brute-force singleton agonists by $E_b(\{j\})$.}
\label{tab:results:e2}
\end{table}

\section{Discussion} \label{sec:discussion}

Across tasks and models, rule-aligned partitions yield the largest aggregate HQ-T and HQ-F counts, while spectral-only splits remain a strong rule-free fallback. In E1, learned rule splits outperform fake-rule controls and exceed random-coverage and spectral-only splits (Table~\ref{tab:results:e1}). This supports H1 operationally: behavioral rules can align $\Dpos$ with localized agonists whose effects are selective for that slice, but weak splitters should be treated as misspecified when the spectral fallback performs better.

The E0 results separate causal localization from symbolic rule fit. Most task--model settings show substantial localized-agonist union coverage and many symbolic anchors. HQ-T support is substantial for arithmetic (114, 12, and 115 anchors for Qwen2-7B, Qwen2-1.5B, and GPT-J) and jailbreaking (49, 81, and 28 anchors in the same model order). HQ-F is larger in aggregate (1210 vs. 399 anchors), showing that many localized flip patterns are expressible by compact OR-combinations over $\Phi$ when all eligible data points are used for selection and scoring. Thus, \AlgoName can identify behavior-sensitive neurons and, for many of them, attach high-quality symbolic descriptions: HQ-T provides held-out support, and HQ-F provides all-fit support.
Moreover, the eligible-coverage scores in Table~\ref{tab:results:e0} are conditional on the behavioral regime being flipped and are better read with the unablated task score. For GPT-J arithmetic, 100\% localized-agonist union coverage in the $(1{\to}0)$ direction corresponds to only about 2\% of the full arithmetic dataset, because the unablated model solves only about 2\% of prompts. 

Reverse-direction coverage reflects the denominator of the opposite regime. In arithmetic, localized-agonist union coverage in the $(0{\to}1)$ direction is measured over baseline errors and increases with baseline accuracy: 11.7\% for GPT-J, 48.6\% for Qwen2-1.5B, and 72.8\% for Qwen2-7B. One hypothesis is that stronger arithmetic models leave more near-miss errors, which are easier to recover by suppressing localized agonists. 
For jailbreaking, the $(1{\to}0)$ direction reaches 97.6--100.0\% coverage across the three jailbreak runs: when the model already produces a jailbreak, at least one localized agonist usually restores refusal. The reverse direction is much weaker, with localized-agonist union coverage in the $(0{\to}1)$ direction equal to 17.3\%, 32.5\%, and 28.7\% for GPT-J, Qwen2-1.5B, and Qwen2-7B. This suggests that existing jailbreak successes depend on localized mechanisms more consistently than safe refusals can be converted into jailbreaks. One possible explanation is that refusal is not simply the inverse of harmful answering: safe responses may be supported by redundant safety mechanisms, by prompt predicates outside the localized-agonist set, or by the absence of a sufficiently active harmful-answering trajectory.

Across the six main runs, baseline task score and localized-agonist union coverage in the $(0{\to}1)$ direction are strongly associated (Pearson $r=0.946$, Spearman $\rho=0.943$). Since task score means arithmetic competence in one setting and jailbreak success in the other, we treat this as a hypothesis about eligible-set composition: reverse flips may be easier when the remaining opposite-regime examples are closer to behavior the model already expresses.

E2 supports a regime-conditional version of H2. When the model has learned the target behavior and the intervention suppresses it, as in correct-to-incorrect arithmetic, we observe the overtopping structure that CHA exploits. In weaker, low-accuracy, or opposite-direction regimes, coarse group effects can be weak even when brute-force singleton ablation finds isolated effects, allowing CHA to miss near-threshold singletons. Overtopping seems therefore an empirical property of particular learned regimes, not a regime-independent law, motivating localized-agonist union coverage in addition to per-neuron effects.

In Qwen2 arithmetic, EAP-IG emphasizes late layers, while exact ablations find many high-strength singleton agonists in early layers, especially \texttt{m3} and \texttt{m4}, with a smaller late-layer peak at \texttt{m26} and \texttt{m27} (Fig.~\ref{fig:e0:layer_coverage_arith_qwen2}). This differs from late-layer-focused arithmetic analyses based on linear probing or LogitLens-style readouts~\citep{nikankin2024arithmetic}. We interpret the gap as a difference between attribution to logit-proximal mediators and necessity under exact ablation. Accordingly, EAP-IG is used only as a candidate reducer; causal claims rely on exact ablation.
Notably, on Qwen2-7B, localized arithmetic and jailbreaking agonists overlap weakly: $9$ shared neurons among $135$ arithmetic and $120$ jailbreaking agonists (Jaccard $0.037$ at $\tau{=}0.2$), indicating that most discovered singletons are task-local rather than global. Finally, the donor-grounded resampling check and the HANS NLI study (Appendix~\ref{app:additional_results}) provide additional supporting evidence beyond the two main tasks.

\section{Limitations} \label{sec:limitations}

\AlgoName produces \emph{task-local}, neuron-anchored explanations by combining behavioral rules with targeted internal interventions, which carries several limitations. Localization depends on the quality of the behavioral splitter $\rzero$ and the vocabulary $\Phi$: if $\rzero$ entangles mechanisms or $\Phi$ misses the true triggers, selectivity may degrade and anchored-rule quality may be low, even when many $\tau$-agonists are localized. The experiments therefore separate HQ-T from HQ-F rule-quality scopes and include spectral-only and fake-rule controls. Anchoring is evidence for an interpretable trigger, not proof that the underlying mechanism is uniquely symbolic.

Stage~2 may miss influential neurons because it shortlists candidates with EAP-IG. We use EAP-IG only as a high-recall reducer, not as a claimed-best circuit discovery method; comparisons with alternatives such as ACDC are outside this paper. To reduce candidate-miss concerns, the experiments retain large MLP-write candidate sets, about $43{,}008$ coordinates in the main arithmetic runs; include an EAP-IG export-cap sensitivity check from $25{,}000$ to $80{,}000$ retained coordinates (Appendix~\ref{app:sensitivity}); and report brute-force singleton baselines over the \emph{same} retained set to separate candidate misses from CHA search error.

Our causal claims are further bounded by the intervention class and its semantics. Although the framework can be instantiated with attention-head, hidden-MLP, routing, or other component-level bases, the present experiments intervene only at MLP writes into the residual stream and, by default, only during decoding; this may miss attention-path mechanisms, mixture-of-experts (MoE) routing mechanisms, multimodal circuits, and effects that act primarily during prefill. Replacing coordinates with zero, mean, or mean-positional baselines is also synthetic and can introduce distribution shift or LayerNorm-mediated side effects. We reduce these concerns with baseline filtering, separate strength and rule-selectivity reporting, and a donor-grounded resampling check (Appendix~\ref{app:donor_baseline}); donor-grounded interventions are more conservative but still share the limits of synthetic replacement.

Our main guarantee is evidence of necessity under ablation, not sufficiency, uniqueness, or minimality. Effect estimates have finite-sample uncertainty because CHA estimates flip rates on compact evaluation subsets. The strongest theorem-style pruning guarantee additionally requires the regime-conditional assumptions in Appendix~\ref{app:cha_validity}, including a valid alpha allocation across queried groups; the experiments use a simpler approximate pruning rule because the theorem-aligned global alpha-spending mode was too conservative in finite samples. We therefore interpret the reported CHA runs as confidence-guided empirical pruning rather than as a fully dependence-aware global certification.

Finally, our empirical scope is limited. The main results cover two tasks and open-weight models up to 7B parameters, with a supplementary HANS NLI study. Generalization to larger models, other architectures, human-judged safety outcomes, multilingual settings, or tasks with distributed mechanisms remains to be established.

\section{Conclusion}
We introduced \AlgoName, a rule-guided, coarse-to-fine intervention pipeline that \textit{(i)} localizes sparse causal neuron activations in a chosen intervention basis and \textit{(ii)} produces neuron-anchored rules that predict \emph{when} those interventions matter.
Across arithmetic and jailbreaking runs on Qwen2-family and GPT-J models, the results suggest that rule-aligned partitions can help isolate selective mechanisms (H1), and that regime-aware contrastive hierarchical ablation recovers high-effect brute-force singleton agonists with substantially fewer evaluations in matched brute-force settings where overtopping holds (H2).
Overall, neuron-anchored rules provide compact symbolic descriptions of behavior-changing interventions and support targeted causal analyses in both flip directions, including cases where a flip degrades the target behavior, improves a task metric, or reveals a safety risk.

\begin{acks}
Supported by Innosuisse (119.321 INT-ICT), HES-SO LH MENA (RPG-2025-139), and Horizon Europe SmartCHANGE (101080965).
\end{acks}

\newpage
\bibliographystyle{ACM-Reference-Format}
\bibliography{references_min}

@inproceedings{lundberg2017unified,
  author = {Lundberg, Scott M. and Lee, Su-In},
  title = {A Unified Approach to Interpreting Model Predictions},
  booktitle = {Advances in Neural Information Processing Systems},
  year = {2017},
  doi = {10.48550/arXiv.1705.07874}
}

@inproceedings{chen2016xgboost,
  author = {Chen, Tianqi and Guestrin, Carlos},
  title = {{XGBoost}: A Scalable Tree Boosting System},
  booktitle = {Proc. 22nd {ACM} {SIGKDD} Int. Conf. on Knowledge Discovery and Data Mining},
  pages = {785--794},
  year = {2016},
  doi = {10.1145/2939672.2939785}
}

@misc{nikankin2024arithmetic,
  author = {Nikankin, Yaniv and Reusch, Anja and Mueller, Aaron and others},
  title = {Arithmetic Without Algorithms: Language Models Solve Math with a Bag of Heuristics},
  year = {2024},
  doi = {10.48550/arXiv.2410.21272}
}

@article{ren2024advancements,
  author = {Ren, Mengchao},
  title = {Advancements and Applications of Large Language Models in Natural Language Processing: A Comprehensive Review},
  journal = {Applied and Computational Engineering},
  volume = {97},
  pages = {55--63},
  year = {2024},
  doi = {10.54254/2755-2721/97/20241406}
}

@article{alenezi2025ai,
  author = {Alenezi, Mamdouh and Akour, Mohammed},
  title = {Ai-driven innovations in software engineering: a review of current practices and future directions},
  journal = {Applied Sciences},
  volume = {15},
  number = {3},
  pages = {1344},
  year = {2025},
  doi = {10.3390/app15031344}
}

@article{zhao2024explainability,
  author = {Zhao, Haiyan and Chen, Hanjie and Yang, Fan and others},
  title = {Explainability for Large Language Models: A Survey},
  journal = {ACM Transactions on Intelligent Systems and Technology},
  volume = {15},
  number = {2},
  pages = {1--38},
  year = {2024},
  doi = {10.1145/3639372}
}

@inproceedings{zilke2016deepred,
  author = {Zilke, Jan Ruben and Loza Menc{\'{\i}}a, Eneldo and Janssen, Frederik},
  title = {{DeepRED}: Rule Extraction from Deep Neural Networks},
  booktitle = {Int. Conf. on Discovery Science},
  pages = {457--473},
  year = {2016},
  doi = {10.1007/978-3-319-46307-0_29}
}

@misc{zarlenga2021efficient,
  author = {Zarlenga, Mateo Espinosa and Shams, Zohreh and Jamnik, Mateja},
  title = {Efficient Decompositional Rule Extraction for Deep Neural Networks},
  year = {2021},
  doi = {10.48550/arXiv.2111.12628}
}

@inproceedings{sovrano2026ruleshap,
  author    = {Sovrano, Francesco},
  title     = {Can Global XAI Methods Reveal Injected Behaviours in LLMs? SHAP vs Rule Extraction vs RuleSHAP},
  booktitle = {Proceedings of the 32nd ACM SIGKDD Conference on Knowledge Discovery and Data Mining},
  year      = {2026},
  numpages = {12},
  doi      = {10.1145/3770855.3818093}
}

@misc{poyiadzi2021understanding,
  author = {Poyiadzi, Rafael and Renard, Xavier and Laugel, Thibault and others},
  title = {Understanding Surrogate Explanations: The Interplay Between Complexity, Fidelity and Coverage},
  year = {2021},
  doi = {10.48550/arXiv.2107.04309}
}

@inproceedings{herbinger2023leveraging,
  author = {Herbinger, Julia and Dandl, Susanne and Ewald, Fiona K. and others},
  title = {Leveraging Model-Based Trees as Interpretable Surrogate Models for Model Distillation},
  booktitle = {Artificial Intelligence. ECAI 2023 International Workshops},
  volume = {1947},
  pages = {232--249},
  year = {2024},
  doi = {10.1007/978-3-031-50396-2_13}
}

@inproceedings{syed2023attrpatching,
  author = {Syed, Aaquib and Rager, Can and Conmy, Arthur},
  title = {Attribution Patching Outperforms Automated Circuit Discovery},
  booktitle = {Proc. 7th BlackboxNLP Workshop: Analyzing and Interpreting Neural Networks for NLP},
  year = {2024},
  doi = {10.18653/v1/2024.blackboxnlp-1.25}
}

@misc{hughes2024best,
  author = {Hughes, John and Price, Sara and Lynch, Aengus and others},
  title = {Best-of-$n$ Jailbreaking},
  year = {2024},
  doi = {10.48550/arXiv.2412.03556}
}

@inproceedings{wang2025modelsurgery,
  author = {Wang, Huanqian and Yue, Yang and Lu, Rui and others},
  title = {Model Surgery: Modulating {LLM}'s Behavior Via Simple Parameter Editing},
  booktitle = {Proc. 2025 Conference of the Nations of the Americas Chapter of the Association for Computational Linguistics: Human Language Technologies (Volume 1: Long Papers)},
  pages = {6337--6357},
  year = {2025},
  doi = {10.18653/v1/2025.naacl-long.321}
}

@inproceedings{zhang2025ece,
  author = {Zhang, Tianyu and Fang, Junfeng and Jiang, Houcheng and others},
  title = {Explainable and Efficient Editing for Large Language Models},
  booktitle = {Proc. {ACM} Web Conference 2025},
  pages = {1963--1976},
  year = {2025},
  doi = {10.1145/3696410.3714835}
}

@misc{yu2025genderbias,
  author = {Yu, Zeping and Ananiadou, Sophia},
  title = {Understanding and Mitigating Gender Bias in {LLM}s via Interpretable Neuron Editing},
  year = {2025},
  doi = {10.48550/arXiv.2501.14457}
}

@misc{zhang2025llmcas,
  author = {Zhang, Jensen and Liu, Ningyuan and Fan, Yijia and others},
  title = {{LLM}-{CAS}: Dynamic Neuron Perturbation for Real-Time Hallucination Correction},
  year = {2025},
  doi = {10.48550/arXiv.2512.18623}
}

@inproceedings{pan2024multimodalneurons,
  author = {Pan, Haowen and Cao, Yixin and Wang, Xiaozhi and others},
  title = {Finding and Editing Multi-Modal Neurons in Pre-Trained Transformers},
  booktitle = {Findings of the Association for Computational Linguistics: ACL 2024},
  pages = {1012--1037},
  year = {2024},
  doi = {10.18653/v1/2024.findings-acl.60}
}

@misc{gu2023neuronpatching,
  author = {Gu, Jian and Aleti, Aldeida and Chen, Chunyang and others},
  title = {Neuron Patching: Semantic-Based Neuron-Level Language Model Repair for Code Generation},
  year = {2023},
  doi = {10.48550/arXiv.2312.05356}
}

@inproceedings{jiang2025nse,
  author = {Jiang, Houcheng and Fang, Junfeng and Zhang, Tianyu and others},
  title = {Neuron-Level Sequential Editing for Large Language Models},
  booktitle = {Proc. 63rd Annual Meeting of the Association for Computational Linguistics (Volume 1: Long Papers)},
  pages = {16678--16702},
  year = {2025},
  doi = {10.18653/v1/2025.acl-long.815}
}

@article{zhou2025editingmemories,
  author = {Zhou, Wei and Wei, Wei and Cao, Guibang and others},
  title = {Editing Memories Through Few Targeted Neurons},
  journal = {Proc. AAAI Conf. on Artificial Intelligence},
  volume = {39},
  number = {24},
  pages = {26111--26119},
  year = {2025},
  doi = {10.1609/aaai.v39i24.34807}
}

@inproceedings{meng2022rome,
  author = {Meng, Kevin and Bau, David and Andonian, Alex and others},
  title = {Locating and Editing Factual Associations in GPT},
  booktitle = {Advances in Neural Information Processing Systems},
  year = {2022},
  doi = {10.48550/arXiv.2202.05262}
}

@inproceedings{lin2025implicit,
  author = {Lin, Tianhe and Xie, Jian and Yuan, Siyu and others},
  title = {Implicit Reasoning in Transformers is Reasoning through Shortcuts},
  booktitle = {Findings of the Association for Computational Linguistics: ACL 2025},
  pages = {9470--9487},
  year = {2025},
  doi = {10.18653/v1/2025.findings-acl.493}
}

@misc{Meng2023MEMIT,
  author = {Meng, Kevin and Sen Sharma, Arnab and Andonian, Alex and others},
  title = {Mass-Editing Memory in a Transformer},
  year = {2023},
  doi = {10.48550/arXiv.2210.07229}
}

@misc{Mitchell2022MEND,
  author = {Mitchell, Eric and Lin, Charles and Bosselut, Antoine and others},
  title = {Fast Model Editing at Scale},
  year = {2022},
  doi = {10.48550/arXiv.2110.11309}
}

@misc{Wang2025SADI,
  author = {Wang, Weixuan and Yang, Jingyuan and Peng, Wei},
  title = {Semantics-Adaptive Activation Intervention for {LLM}s via Dynamic Steering Vectors},
  year = {2025},
  doi = {10.48550/arXiv.2410.12299}
}

@misc{conmy2023acdc,
  author = {Arthur Conmy and Augustine N. Mavor{-}Parker and Aengus Lynch and others},
  title = {Towards Automated Circuit Discovery for Mechanistic Interpretability},
  year = {2023},
  doi = {10.48550/arXiv.2304.14997}
}

@inproceedings{dai2022knowledgeneurons,
  author = {Dai, Damai and Dong, Li and Hao, Yaru and others},
  title = {Knowledge Neurons in Pretrained Transformers},
  booktitle = {Proc. 60th Annual Meeting of the Association for Computational Linguistics (Volume 1: Long Papers)},
  pages = {8493--8502},
  year = {2022},
  doi = {10.18653/v1/2022.acl-long.581}
}

@misc{hanna2024have,
  author = {Hanna, Michael and Pezzelle, Sandro and Belinkov, Yonatan},
  title = {Have Faith in Faithfulness: Going Beyond Circuit Overlap When Finding Model Mechanisms},
  year = {2024},
  doi = {10.48550/arXiv.2403.17806}
}

@misc{huang2023transformer,
  author = {Huang, Zeyu and Shen, Yikang and Zhang, Xiaofeng and others},
  title = {Transformer-Patcher: One Mistake Worth One Neuron},
  year = {2023},
  doi = {10.48550/arXiv.2301.09785}
}

@inproceedings{vaswani2017attention,
  author = {Vaswani, Ashish and Shazeer, Noam and Parmar, Niki and others},
  title = {Attention Is All You Need},
  booktitle = {Advances in Neural Information Processing Systems},
  year = {2017},
  doi = {10.48550/arXiv.1706.03762}
}

@article{hwang1972method,
  author = {Hwang, Frank K.},
  title = {A Method for Detecting All Defective Members in a Population by Group Testing},
  journal = {Journal of the American Statistical Association},
  volume = {67},
  number = {339},
  pages = {605--608},
  year = {1972},
  doi = {10.1080/01621459.1972.10481257}
}

@book{du2000grouptesting,
  author = {Du, Dingzhu and Hwang, Frank},
  title = {Combinatorial Group Testing and Its Applications},
  year = {2000},
  doi = {10.1142/4252}
}

@inproceedings{DBLP:conf/sigmod/AgrawalIS93,
  author = {Rakesh Agrawal and Tomasz Imielinski and Arun N. Swami},
  title = {Mining Association Rules between Sets of Items in Large Databases},
  booktitle = {Proc. 1993 {ACM} {SIGMOD} Int. Conf. on Management of Data},
  pages = {207--216},
  year = {1993},
  doi = {10.1145/170035.170072}
}

@article{clopper1934use,
  author = {Clopper, Charles J. and Pearson, Egon S.},
  title = {The Use of Confidence or Fiducial Limits Illustrated in the Case of the Binomial},
  journal = {Biometrika},
  volume = {26},
  number = {4},
  pages = {404--413},
  year = {1934},
  doi = {10.1093/biomet/26.4.404}
}

@article{nemhauser1978analysis,
  author = {Nemhauser, George L. and Wolsey, Laurence A. and Fisher, Marshall L.},
  title = {An Analysis of Approximations for Maximizing Submodular Set Functions---I},
  journal = {Mathematical Programming},
  volume = {14},
  number = {1},
  pages = {265--294},
  year = {1978},
  doi = {10.1007/BF01588971}
}

@book{fujishige2005submodular,
  author = {Fujishige, Satoru},
  title = {Submodular Functions and Optimization},
  volume = {58},
  year = {2005}
}

@misc{ba2016layernorm,
  author = {Ba, Jimmy Lei and Kiros, Jamie Ryan and Hinton, Geoffrey E.},
  title = {Layer Normalization},
  year = {2016},
  doi = {10.48550/arXiv.1607.06450}
}

@misc{hendrycks2016gelu,
  author = {Hendrycks, Dan and Gimpel, Kevin},
  title = {Gaussian Error Linear Units (GELUs)},
  year = {2016},
  doi = {10.48550/arXiv.1606.08415}
}

@article{gonzalez1985clustering,
  author = {Gonzalez, Teofilo F},
  title = {Clustering to minimize the maximum intercluster distance},
  journal = {Theoretical computer science},
  volume = {38},
  pages = {293--306},
  year = {1985},
  doi = {10.1016/0304-3975(85)90224-5}
}

@article{chicco2020advantages,
  author = {Chicco, Davide and Jurman, Giuseppe},
  title = {The advantages of the Matthews correlation coefficient (MCC) over F1 score and accuracy in binary classification evaluation},
  journal = {BMC genomics},
  volume = {21},
  number = {1},
  pages = {6},
  year = {2020},
  doi = {10.1186/s12864-019-6413-7}
}

@inproceedings{liu1998integrating,
  author = {Liu, Bing and Hsu, Wynne and Ma, Yiming},
  title = {Integrating Classification and Association Rule Mining},
  booktitle = {Proc. Fourth Int. Conf. on Knowledge Discovery and Data Mining},
  year = {1998},
  url = {https://aaai.org/papers/00080-KDD98-012-integrating-classification-and-association-rule-mining/}
}

@inproceedings{cohen1995fast,
  author = {William W. Cohen},
  title = {Fast Effective Rule Induction},
  booktitle = {Proc. 12th Int. Conf. on Machine Learning},
  pages = {115--123},
  year = {1995},
  doi = {10.1016/B978-1-55860-377-6.50023-2}
}

@article{letham2015interpretable,
  author = {Benjamin Letham and Cynthia Rudin and Tyler H. McCormick and others},
  title = {Interpretable classifiers using rules and Bayesian analysis: Building a better stroke prediction model},
  journal = {The Annals of Applied Statistics},
  volume = {9},
  number = {3},
  pages = {1350--1371},
  year = {2015},
  doi = {10.1214/15-AOAS848}
}

@inproceedings{angelino2017learning,
  author = {Angelino, Elaine and Larus-Stone, Nicholas and Alabi, Daniel and others},
  title = {Learning Certifiably Optimal Rule Lists},
  booktitle = {Proc. 23rd ACM SIGKDD Int. Conf. on Knowledge Discovery and Data Mining},
  year = {2017},
  doi = {10.1145/3097983.3098047}
}

@article{aldridge2019group,
  author = {Matthew Aldridge and Oliver Johnson and Jonathan Scarlett},
  title = {Group testing: an information theory perspective},
  journal = {Foundations and Trends in Communications and Information Theory},
  volume = {15},
  number = {3--4},
  pages = {196--392},
  year = {2019},
  doi = {10.1561/0100000099}
}

@misc{bills2023language,
  author = {Bills, Steven and Cammarata, Nick and Mossing, Dan and others},
  title = {Language models can explain neurons in language models},
  year = {2023},
  url = {https://openaipublic.blob.core.windows.net/neuron-explainer/paper/index.html}
}

@misc{nostalgebraist2020logitlens,
  author = {Belrose, Nora and Furman, Zach and Smith, Logan and others},
  title = {Eliciting Latent Predictions from Transformers with the Tuned Lens},
  year = {2023},
  doi = {10.48550/arXiv.2303.08112}
}

@misc{nanda2022transformerlens,
  author = {Neel Nanda and Joseph Bloom},
  title = {TransformerLens},
  year = {2022},
  url = {https://github.com/TransformerLensOrg/TransformerLens}
}

@inproceedings{mccoy2019right,
  author = {McCoy, R. Thomas and Pavlick, Ellie and Linzen, Tal},
  title = {Right for the Wrong Reasons: Diagnosing Syntactic Heuristics in Natural Language Inference},
  booktitle = {Proc. 57th Annual Meeting of the Association for Computational Linguistics},
  pages = {3428--3448},
  year = {2019},
  doi = {10.18653/v1/P19-1334}
}

@misc{yang2024qwen2,
  author = {Yang, An and Yang, Baosong and Hui, Binyuan and others},
  title = {{Qwen2} Technical Report},
  year = {2024},
  doi = {10.48550/arXiv.2407.10671}
}

@software{wang2021gptj,
  author = {Wang, Ben and Komatsuzaki, Aran},
  title = {{GPT-J-6B}: A 6 Billion Parameter Autoregressive Language Model},
  year = {2021},
  url = {https://huggingface.co/EleutherAI/gpt-j-6b}
}

\appendix

\section{CHA validity, overtopping effects, and pruning proof}
\label{app:cha_validity}

This appendix uses the notation from \S\ref{sec:problem_setup}. A baseline regime $b$ fixes the flip direction, $\Dpos/\Dneg$ are the rule-induced slices, and $E_b(A)=\max\{\delta_{+\mid b}(A),\delta_{-\mid b}(A)\}$ is strength. A singleton $j$ is a $\tau$-agonist when $E_b(\{j\})\ge\tau$ and $\Dpos$-selective when $\delta_{+\mid b}(\{j\})-\delta_{-\mid b}(\{j\})\ge\varepsilon$.
\emph{Overtopping} denotes the dominance/overlap pattern that enables the CHA search.

\textbf{Dominance, overlap, and examples.}
For a group $A$ with $E_b(A)>0$, define
$\mathrm{Dom}^{(m)}_b(A)=\max_{B\subseteq A,\,1\le |B|\le m}E_b(B)/E_b(A)$.
The group is $(m,\rho)$-overtopped when this ratio is at least $\rho$, meaning that a small subgroup explains most of its measured strength. Overlap occurs when neurons flip many of the same prompts, so their bounded-probability effects need not add linearly (see Figure~\ref{fig:overtopping_schematic}).
Table~\ref{tab:app:overtopping_examples} gives edge cases: \texttt{m4} shows single-neuron dominance, \texttt{m26} shows overlapping large singleton effects, and \texttt{m12} shows reverse-regime dominance. 
%
%
The cutoff $\tau=0.2$ is operational: it exceeds small finite-sample fluctuations on compact subsets while retaining near-threshold effects for sensitivity analysis.

\begin{table}[bh]
\centering
\resizebox{.7\columnwidth}{!}
{%
\begin{tabular}{llrrrrr}
\toprule
Case & node & depth & size & $a_+$ & $a_-$ & $\widehat U_E$ \\
\midrule
\multicolumn{7}{l}{\textit{m4, positive: early dominance}}\\
parent & group & 3 & 475 & 0.01 & 0.00 & 1.00 \\
G1 & subgroup & 4 & 237 & 0.84 & 0.81 & 0.30 \\
G2 & subgroup & 4 & 238 & 0.00 & 0.01 & 1.00 \\
\textbf{G2} & \textbf{leaf}$^\dagger$ & 12 & 1 & 0.03 & 0.00 & \textbf{1.00} \\
\textbf{G1} & \textbf{leaf} & 12 & 1 & 0.93 & 0.85 & \textbf{0.25} \\
\textbf{G1} & \textbf{leaf} & 12 & 1 & 0.98 & 0.89 & \textbf{0.21} \\
\midrule
\multicolumn{7}{l}{\textit{m26, positive: late overlap}}\\
parent & group & 4 & 231 & 0.96 & 0.60 & 0.52 \\
\textbf{A} & \textbf{overlap leaf} & 12 & 1 & 0.96 & 0.81 & \textbf{0.30} \\
\textbf{B} & \textbf{overlap leaf} & 12 & 1 & 0.96 & 0.79 & \textbf{0.32} \\
parent & group & 2 & 860 & 0.15 & 0.39 & 0.92 \\
\textbf{C} & \textbf{leaf} & 12 & 1 & 0.90 & 0.54 & \textbf{0.58} \\
\textbf{D} & \textbf{leaf}$^\dagger$ & 12 & 1 & 0.26 & 0.43 & \textbf{0.83} \\
\midrule
\multicolumn{7}{l}{\textit{m12, negative: reverse-regime dominance}}\\
parent & group & 4 & 237 & 0.35 & 0.07 & 0.48 \\
\textbf{leaf} & \textbf{leaf}$^\dagger$ & 12 & 1 & 0.29 & 0.07 & \textbf{0.42} \\
\bottomrule
\end{tabular}%
}
\caption{Qwen2-7B arithmetic overtopping examples. $a_+$/$a_-$ are post-intervention accuracies. Leaves have $\widehat U_E \ge 0.2$, while omitted leaves do not; $\dagger$ marks dominance.}
\label{tab:app:overtopping_examples}
\end{table}

\begin{figure}[tbh]
\centering
\includegraphics[width=\linewidth]{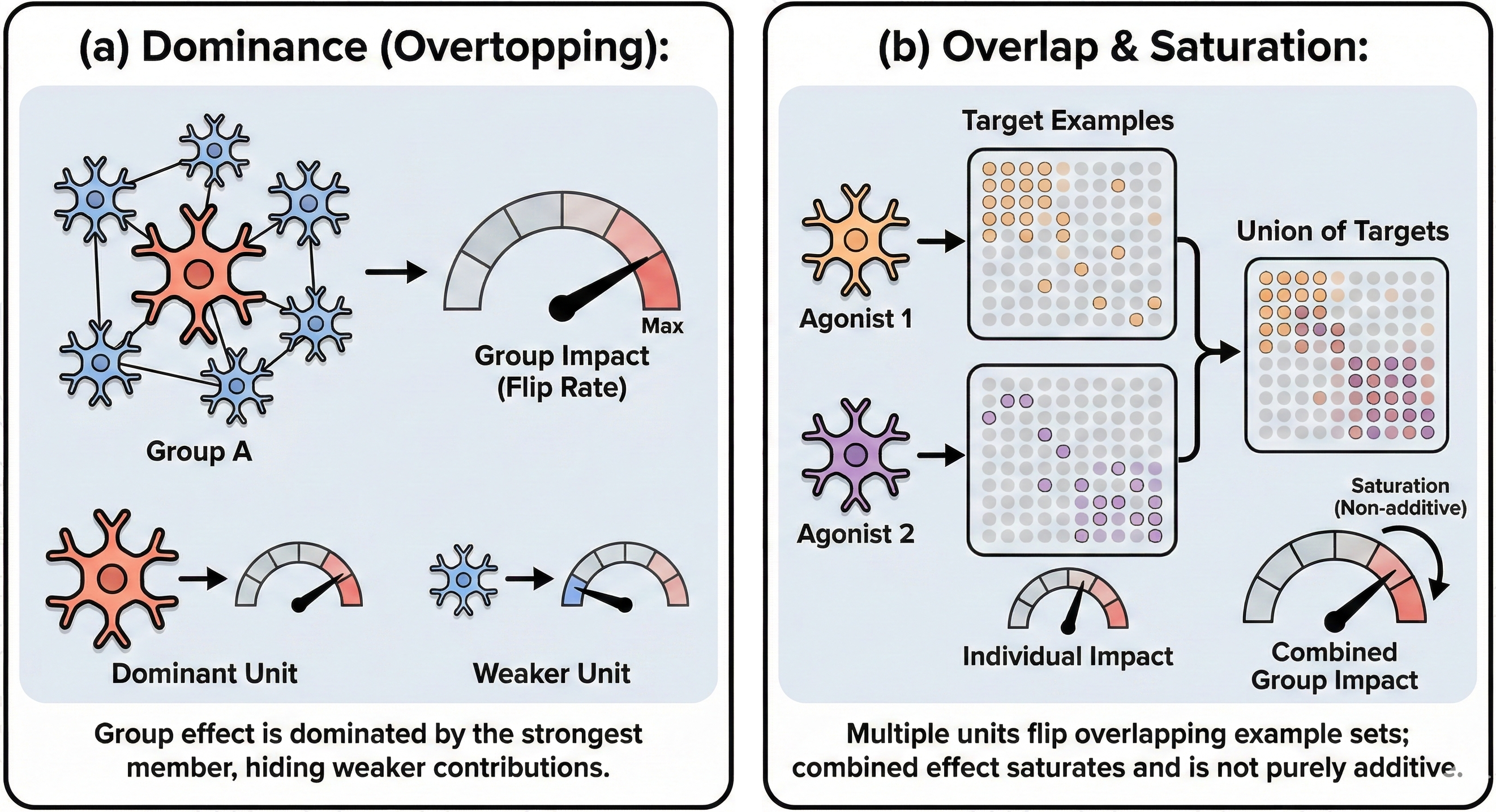}
\caption{Dominance and overlap in hierarchical ablation: singletons can overtop groups, while overlapping prompt flips make effects non-additive.}
\Description{Schematic of hierarchical ablation. A parent group splits into subgroups; one branch contains a dominant or catastrophic unit, while another branch contains several partially overlapping agonists.}
\label{fig:overtopping_schematic}
\end{figure}

\textbf{Pruning proof.}
Let $U$ be a finite candidate set, let $\mathcal{T}$ be the full binary partition tree over $U$, and let $H_{\tau,b}=\{j\in U:E_b(\{j\})\ge\tau\}$ be the true singleton agonists in regime $b$. CHA estimates $E_b(A)$ for queried nodes $A$, computes a UCB $\widehat U_E(A)$, prunes only when $\widehat U_E(A)<\tau$, and otherwise descends to the children.

\begin{assumption}[Regime-fixed evaluation]
All queried nodes use the same baseline regime $b$ and definitions of $\delta_{+\mid b}$, $\delta_{-\mid b}$, and $E_b$.
\end{assumption}
\begin{assumption}[Group-level no false negatives]
For every node $A \!\in\! \mathcal{T}$, if $A \!\cap\! H_{\tau,b} \!\ne\! \emptyset$, then $E_b(A) \!\ge\! \tau$. This is the monotone-overtopping condition: a group containing a true $\tau$-agonist remains above threshold. When this condition fails, CHA can prune a branch that contains singleton effects.
\end{assumption}


\begin{assumption}[Valid confidence budget]
For each queried node $A$, conditional on the prior query history, $\widehat U_E(A)$ satisfies
$
\Pr\{E_b(A)>\widehat U_E(A)\mid A\text{ queried},\mathcal{H}_A\}\le\alpha_A,
$
where $\alpha_A$ is predictable and $\sum_{A\in\mathcal{T}}\alpha_A\le\alpha_{\mathrm{tot}}$.
\end{assumption}


For the two-slice strength statistic used here, a valid upper confidence bound can be obtained by allocating alpha to each slice, computing one-sided binomial upper bounds for $\delta_{+\mid b}(A)$ and $\delta_{-\mid b}(A)$, and taking their maximum. The proof below abstracts this construction into a single bound $\widehat U_E(A)$; it therefore applies to any procedure whose per-node bounds satisfy the stated alpha-allocation condition.

\begin{theorem}[Pruning safety of CHA]
Under the three assumptions above, with probability at least $1-\alpha_{\mathrm{tot}}$, CHA does not prune nodes containing a true regime-$b$ $\tau$-agonist. Therefore every true $\tau$-agonist leaf is queried.
\end{theorem}
\begin{proof}
For each tree node $A$, let $Q_A$ be the event that $A$ is queried and let $F_A=Q_A\cap\{E_b(A)>\widehat U_E(A)\}$. By the conditional confidence assumption and predictable choice of $\alpha_A$, $\Pr(F_A)\le\alpha_A$ after averaging over histories. A union bound over the finite tree gives $\Pr(\cup_A F_A)\le\sum_A\alpha_A\le\alpha_{\mathrm{tot}}$. On the complement, every queried upper bound dominates the corresponding true strength. If a queried node $A$ contains a true agonist, group-level no false negatives gives $E_b(A)\ge\tau$, hence $\widehat U_E(A)\ge E_b(A)\ge\tau$. Since CHA prunes only when $\widehat U_E(A)<\tau$, $A$ is not pruned. Induction along the root-to-leaf path reaches every true agonist leaf.
\end{proof}


If the optional separation condition also holds, namely $E_b(A)<\tau$ whenever $A\cap H_{\tau,b}=\emptyset$, then the strength predicate behaves as an exact group-testing oracle. With balanced splits and exact tests, the queried-node count is $O(k\log(N/k)+k)$ for $k=|H_{\tau,b}|$. Without separation, pruning remains safe under the theorem but the search can be larger because non-agonist subtrees may still test positive. The theorem proves reachability of true leaves under stated assumptions; it does not prove sufficiency, uniqueness, minimality, or recovery when overtopping fails.

\section{Operational settings}
\label{app:operational_settings}

\textbf{LLM-assisted predicate extraction.}
A task record contains the prompt, available parsed fields, model output, task label, and baseline-output predicates. We build $\Phi$ by starting from deterministic task-specific seed predicates (e.g., arithmetic operators, operand ranges), sampling contrastive success/failure records, asking an LLM agent to propose candidate human-readable Boolean predicates that may separate the two record groups, requiring accepted proposals to be deterministic record-level predicates, and scoring retained predicates on validation records using AUC, average precision above base rate, and standardized mean separation. Ambiguous, non-computable, or duplicate proposals are discarded. 
The LLM only proposes predicate text; accepted predicates must be deterministic Boolean tests before mechanistic evaluation, and the predicate-generation prompt and proposer model are fixed within each task. Low-signal predicates, correlated duplicates, and MAD outliers are filtered before RuleSHAP. Splitters and anchored rules are selected by MCC and importance, then greedily compactified into OR-combinations that keep only clauses improving rulebase-level MCC; HQ-T and HQ-F use the scoring scopes defined in \S\ref{sec:method:anchoring}.

\textbf{Compression, candidate reduction, and ablation.}
Spectral compression embeds prompts with the final normalized hidden state \RepHook{} at the last prompt token, projects to $32$ PCA components, and selects representatives by greedy $k$-center coverage with radius $0.5$. Circuit discovery uses up to $128$ matched associated/unrelated pairs per circuit, with nearest-neighbor matching in spectral space. Group ablation uses up to $64$ examples per slice, selected as compact baseline-specific associated/unrelated subsets.
EAP-IG serves only as a candidate reducer; exact ablation is the causal estimator. Unless stated otherwise, EAP-IG ranks by absolute attribution, mean-aggregates over the answer span, and uses MLP-write candidates. CHA runs layerwise over MLP-write coordinates with $\tau=0.2$ and node-level pruning confidence $0.05$, using one-sided Clopper--Pearson UCBs on flip rates.

\textbf{Replacement baselines.}
The main ablation is mean-positional replacement, estimated from $256$ prompts. Alternatives are zero, mean, mean-donor, and mean-donor-positional. Where specified, interventions are decode-only, so prompt prefill is fixed and ablations act during answer generation. Unless stated otherwise, HQ rule-bearing agonists require $\mathrm{MCC}\ge0.70$ and dataset coverage at least $0.005$ ($0.5\%$).

\section{Additional experimental details}
\label{app:additional_results}

\textbf{Attribution vs.\ exact ablation.}\label{app:eap_vs_cha_qwen2_arith}
For arithmetic/Qwen2-7B, the top-$50$ EAP-IG coordinates contain $0$ coordinates in early layers \texttt{m3}--\texttt{m4} and $37$ coordinates in m27. The top-$50$ exact CHA singleton effects contain $42$ coordinates in \texttt{m3}--\texttt{m4} and $4$ coordinates in \texttt{m26}--\texttt{m27}.

\textbf{Sensitivity.}\label{app:sensitivity}
The sensitivity study on arithmetic/Qwen2-1.5B varies the CHA strength threshold $\tau$ and the upstream EAP-IG export cap. 
As shown in Fig.~\ref{fig:app:sensitivity_compact}, increasing $\tau$ from $0.2$ to $0.5$ makes localization stricter: rule-bearing agonists drop from $37$ to $15$. 
HQ-F counts are stable at the main cutoff: $10$ anchors reach MCC $\ge0.70$ for every $\tau$ setting, while one anchor reaches $0.80$ and none reaches $0.85$.
At fixed $\tau=0.2$, upstream EAP-IG export caps of $25$k, $43$k, and $80$k yield $24$, $37$, and $34$ rule-bearing agonists, and $6$, $10$, and $10$ HQ-F anchors at MCC $\ge0.70$. 
Fig.~\ref{fig:app:sensitivity_compact} also shows that increasing $\tau$ reduces mean CHA group evaluations from $824$ to $225$, corresponding to a CHA cost decrease from $1.0\%$ to $0.3\%$ of exhaustive singleton ablation.

\begin{figure}[t]
\centering
\includegraphics[width=\linewidth]{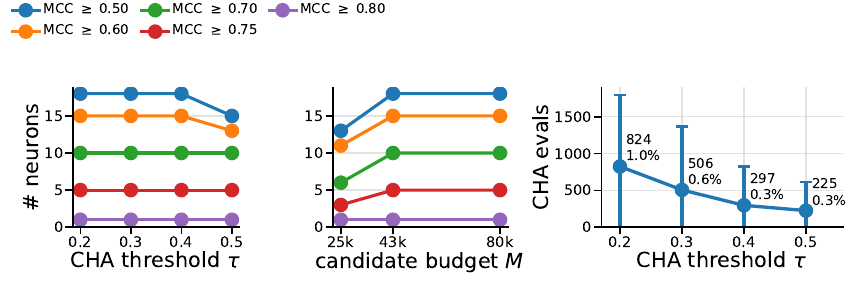}
\Description{The first plot varies effect threshold tau at fixed EAP-IG export cap. The second varies EAP-IG export cap at fixed tau. The third shows mean CHA group-evaluation counts with error bars and point labels giving raw counts and percentages of exhaustive singleton-ablation cost.}
\caption{Sensitivity for arithmetic/Qwen2-1.5B. Left/middle: HQ-F counts across MCC thresholds, varying $\tau$ and retained MLP-write budget. Right: mean CHA cost as the \% of ablations performed by CHA, varying $\tau$.}
\label{fig:app:sensitivity_compact}
\end{figure}

\textbf{Intervention strategies.}\label{app:donor_baseline}
Mean-donor replacement selects, for each requested coordinate, an observed activation value closest to that coordinate's dataset mean. When several coordinates are ablated together, the replacement coordinates need not come from the same prompt or position; the replacement is coordinatewise donor-grounded. 
In arithmetic/Qwen2-7B, mean-donor replacement identifies fewer rule-bearing agonists than mean-positional replacement ($50$ vs. $135$), with a slightly higher HQ-F share ($19/50=38.0\%$ vs. $41/135=30.4\%$). Figure~\ref{fig:app:donor_baseline_compact} shows that both interventions reach $100\%$ coverage for the degradation direction ($1{\to}0$). The improvement direction ($0{\to}1$) is lower, with HQ-F mean-donor anchors covering $56.4\%$ of eligible points versus $50.1\%$ for mean-positional anchors.

\begin{figure}
\centering
\includegraphics[width=\linewidth]{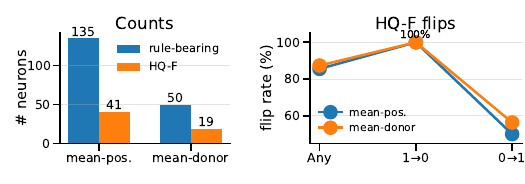}
\Description{Two horizontal panels compare mean-positional replacement and mean-donor replacement. The left panel shows rule-bearing and HQ-F counts. The right panel shows HQ-F flip prevalence for any flip, correct to incorrect, and incorrect to correct.}
\caption{Mean vs. mean-donor for arithmetic/Qwen2-7B. Left: rule-bearing and HQ-F counts. Right: HQ-F flip rates; $1\to0$ and $0\to1$ are the two flip directions.}
\label{fig:app:donor_baseline_compact}
\end{figure}

\textbf{HANS NLI} \cite{mccoy2019right} tests lexical-overlap, subsequence, and constituent heuristics on $1024$ balanced examples, with $B(x)=1$ for a correct generated label. Baseline accuracy is $61.5\%$ for Qwen2-7B and $32.3\%$ for Qwen2-1.5B, and $0.0\%$ for GPT-J. 
With \textit{Rule split + spectral coverage}, Qwen2-7B finds $29$ all-fit rule-bearing agonists, including $23$ HQ-T and $23$ HQ-F anchors, with $100.0\%/48.2\%$ union coverage in the $1{\to}0$ and $0{\to}1$ directions. Qwen2-1.5B finds $26$ all-fit rule-bearing agonists, including $24$ HQ-T and $25$ HQ-F anchors, with $99.7\%/19.8\%$ union coverage.
The runs on the HANS NLI tasks apply interventions over the full input and output sequence, so anchors may include mechanisms active during prompt encoding as well as label generation. This is because output-only CHA interventions did not localize reliable agonists, suggesting dependence on input processing.

\end{document}